\documentclass[runningheads]{llncs}
\usepackage[T1]{fontenc}
\usepackage{graphicx}
\usepackage{booktabs}
\usepackage[misc]{ifsym}
\newcommand{\corr}{(\Letter)}
\usepackage{microtype}
\usepackage{graphicx}
\usepackage{subfigure}
\usepackage{booktabs} % for professional tables
\usepackage[colorlinks=True]{hyperref}
\usepackage{algorithm}

\usepackage{algpseudocode}

\usepackage{colortbl}

\usepackage{amsmath}
\usepackage{amssymb}
\usepackage{mathtools}
\usepackage[capitalize,noabbrev]{cleveref}
\usepackage{colortbl}
\usepackage{nicematrix}
\usepackage[inline, shortlabels]{enumitem}
\usepackage{multirow}

\usepackage{amsmath,amsfonts,bm}
\usepackage{nicefrac}
\usepackage{array}
\usepackage{wrapfig}

\usepackage{hyperref}
\usepackage{url}

\definecolor{bluex}{rgb}{0.27, 0.42, 0.81}
\definecolor{purplex}{HTML}{9564bf}
\definecolor{Gray}{gray}{0.94}

\newtheorem{assumption}[theorem]{Assumption}

\DeclareMathOperator*{\argmin}{arg\,min}

\newcommand{\bE}{\mathbb{E}}
\newcommand{\bR}{\mathbb{R}}

\newcommand{\STAB}[1]{\begin{tabular}{@{}c@{}}#1\end{tabular}}

\newcommand{\eg}{e.g.}

\newcommand{\hB}{\mathcal{B}}

\newcommand{\hD}{\mathcal{D}}

\newcommand{\hL}{\mathcal{L}}

\newcommand{\hO}{\mathcal{O}}

\newcommand{\va}{{\bf a}}
\newcommand{\vb}{{\bf b}}
\newcommand{\vc}{{\bf c}}

\newcommand{\vg}{{\bf g}}

\newcommand{\vv}{{\bf v}}

\newcommand{\vx}{{\bf x}}

\newcommand{\vH}{{\bf H}}
\newcommand{\vI}{{\bf I}}

\newcommand{\vT}{{\bf T}}

\newcommand{\vX}{{\bf X}}

\DeclareMathOperator{\diag}{diag}

\newcommand{\vtheta}{{\boldsymbol \theta}}

\newcommand{\vepsilon}{{\boldsymbol \epsilon}}

\definecolor{darkyellow}{HTML}{D5BA82}
\definecolor{darkgreen}{HTML}{A1D0C7}

\definecolor{darkred}{HTML}{CE5858}
\definecolor{darkblue}{HTML}{5479BD}

\hypersetup{linkcolor=black}
\hypersetup{citecolor=[HTML]{1663A9}}

\begin{document}

\title{Enhancing Sharpness-Aware Minimization by Learning Perturbation Radius}

% \titlerunning{Underwater Basket Weaving Under Extreme Pressure}
% If the full title of your paper is short enough to also fit in the running head, you can omit the abbreviated paper title here. You can check as follows: if you comment out the \titlerunning line, something will appear in the header of all odd-numbered pages of your PDF from page 3 onward. This something is either the full title (in which case all is well), or the error message "Title Suppressed Due to Excessive Length". If this error message appears, you're going to want to provide an abbreviated title within the \titlerunning command, because if you won't do it, Springer will do it for you.

%N.B.: Author information (both in the \author{} and \authorrunning{} command) should only be present in the Camera-Ready Version of your paper. The version that you initially submit for review, ought to be double-blind. So, when initially submitting your paper, use:
%\author{Author information scrubbed for double-blind reviewing}
% \author{Andr\'e Lauren Benjamin\inst{1} \and
% Calvin Cordozar Broadus Jr.\inst{2,3} \corr \and
% Antwan Andr\'e Patton\inst{1}\orcidID{0000-1111-2222-3333}}
\author{Xuehao Wang\inst{1}$^{,\dag}$ \and
Weisen Jiang\inst{1,2}$^{,\dag}$ \and
Shuai Fu\inst{1} \and
Yu Zhang\inst{1} \corr
}

\toctitle{Enhancing Sharpness-Aware Minimization by Learning Perturbation Radius}
\tocauthor{Xuehao~Wang, Weisen~Jiang,Shuai~Fu,Yu~Zhang}

% You may leave out the orcidID information, if you want to.
% Use \corr to indicate the corresponding author. Note the spacing around the \corr command. Only one author can be the corresponding author.

%N.B.: comment out the \authorrunning{} command for the double-blind version of your paper submitted for review. Later, if your paper is accepted, use the command for the Camera-Ready Version.
\authorrunning{X. Wang et al.}
% If there are two authors, write 'A.L. Benjamin and C.C. Broadus Jr.'
% If there are more than two authors, '[...] et al.' is used.

% \institute{Fictional Southern University, Savannah GA 31404, USA \email{\{a.l.benjamin,a.a.patton\}@fsu.fake}
% \and
% Fictional West Coast University, Long Beach CA 90840, USA \email{ccb@fwcu.fake}
% \and
% Secondary European Affiliation, Tiergartenstr. 17, 69121 Heidelberg, Germany
% \email{lncs@springer.com}}

\institute{Department of Computer Science and Engineering, Southern University of Science and Technology, Shenzhen, China
\and
The Hong Kong University of Science and Technology, Hong Kong, China
\\
\email{\{xuehaowangfi,waysonkong,fus.jayce,yu.zhang.ust\}@gmail.com}
}

\renewcommand{\thefootnote}{\dag}
\footnotetext[1]{Equal Contribution.}

\maketitle

\begin{abstract}
Sharpness-aware minimization (SAM) is to improve model generalization by searching for flat minima in the loss landscape. 
The SAM update consists of one step for computing the perturbation and the other for computing the update gradient. Within the two steps, the choice of the perturbation radius is crucial to the performance of SAM, but finding an appropriate perturbation radius is challenging. In this paper, we propose a bilevel optimization framework called LEarning the perTurbation radiuS (LETS) to learn the perturbation radius for sharpness-aware minimization algorithms. 
Specifically, in the proposed LETS method, the upper-level problem aims at seeking a good perturbation radius by minimizing the squared generalization gap between the training and validation losses, while the lower-level problem is the SAM optimization problem.
Moreover, the LETS method can be combined with any variant of SAM. Experimental results on various architectures and benchmark datasets in computer vision and natural language processing demonstrate the effectiveness of the proposed LETS method in improving the performance of SAM.

\keywords{Sharpness-Aware Minimization  \and Bilevel Optimization \and Hyperparameter Optimization}
\end{abstract}

\section{Introduction}
Deep neural networks have demonstrated remarkable performance across various fields~\cite{he2016deep,zagoruyko2016wide},
but they tend to overfit on the training data with poor ability of generalization due to overparameterization~\cite{zhang2021understanding}.
The loss function landscape is intricate and non-linear, characterized by numerous local minima with varying generalization capabilities. 
Several
studies~\cite{hochreiter1994simplifying,keskar2017on,Jiang2020Fantastic} have explored the connection between the geometry of the loss function surface and the generalization ability of neural networks, and have revealed that flatter minima tend to result in better generalization performance than sharper minima~\cite{karolina2017,petzka2021relative,keskar2017on,Jiang2020Fantastic}. 

Sharpness-aware minimization (SAM)~\cite{foret2021sharpness} is an optimization method that solves a min-max optimization problem to seek flat minima.
% to seek flat minima by solving a min-max optimization problem. 
Specifically, SAM aims to find a model parameterized by $\vtheta$ such that its neighbors in parameter space also have good performance.
SAM first computes the worst-case perturbation $\vepsilon$
that maximizes the training loss within 
a neighborhood specified by a perturbation radius $\rho$, and then minimizes the training loss w.r.t. the perturbed model $\vtheta+\vepsilon$.
Many variants of SAM are proposed to improve its
effectiveness~\cite{zhuang2022surrogate,kwon2021asam,zhao2022penalizing,liu2022random} and efficiency~\cite{liu2022towards,zhao2022ss,jiang2023adaptive}.

\begin{figure}
\centering
\subfigure[SAM. \label{fig:mrpc_sam}]{\includegraphics[width=0.23\textwidth]{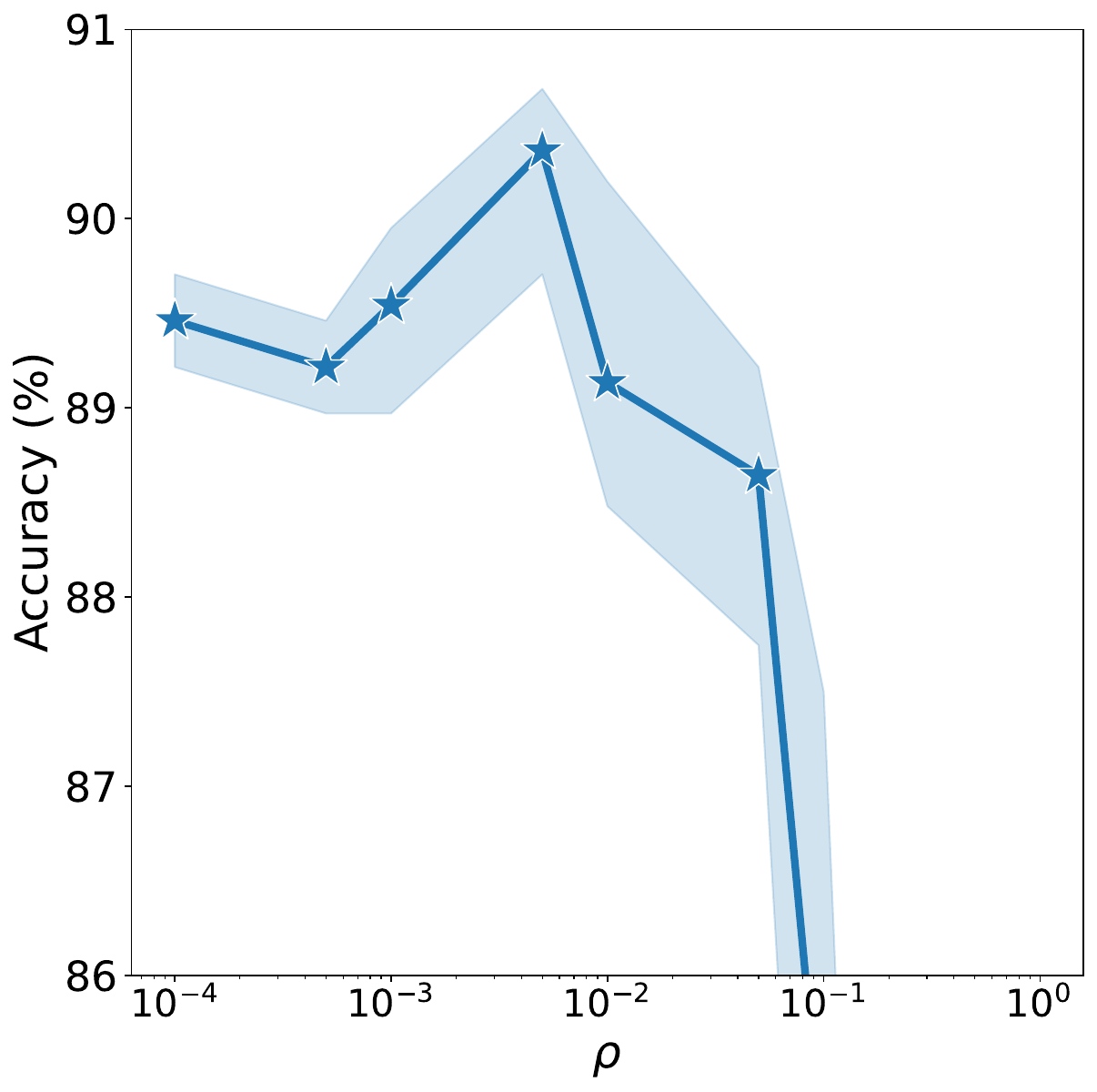}}
\quad
\subfigure[ASAM. \label{fig:mrpc_asam}]{\includegraphics[width=0.23\textwidth]{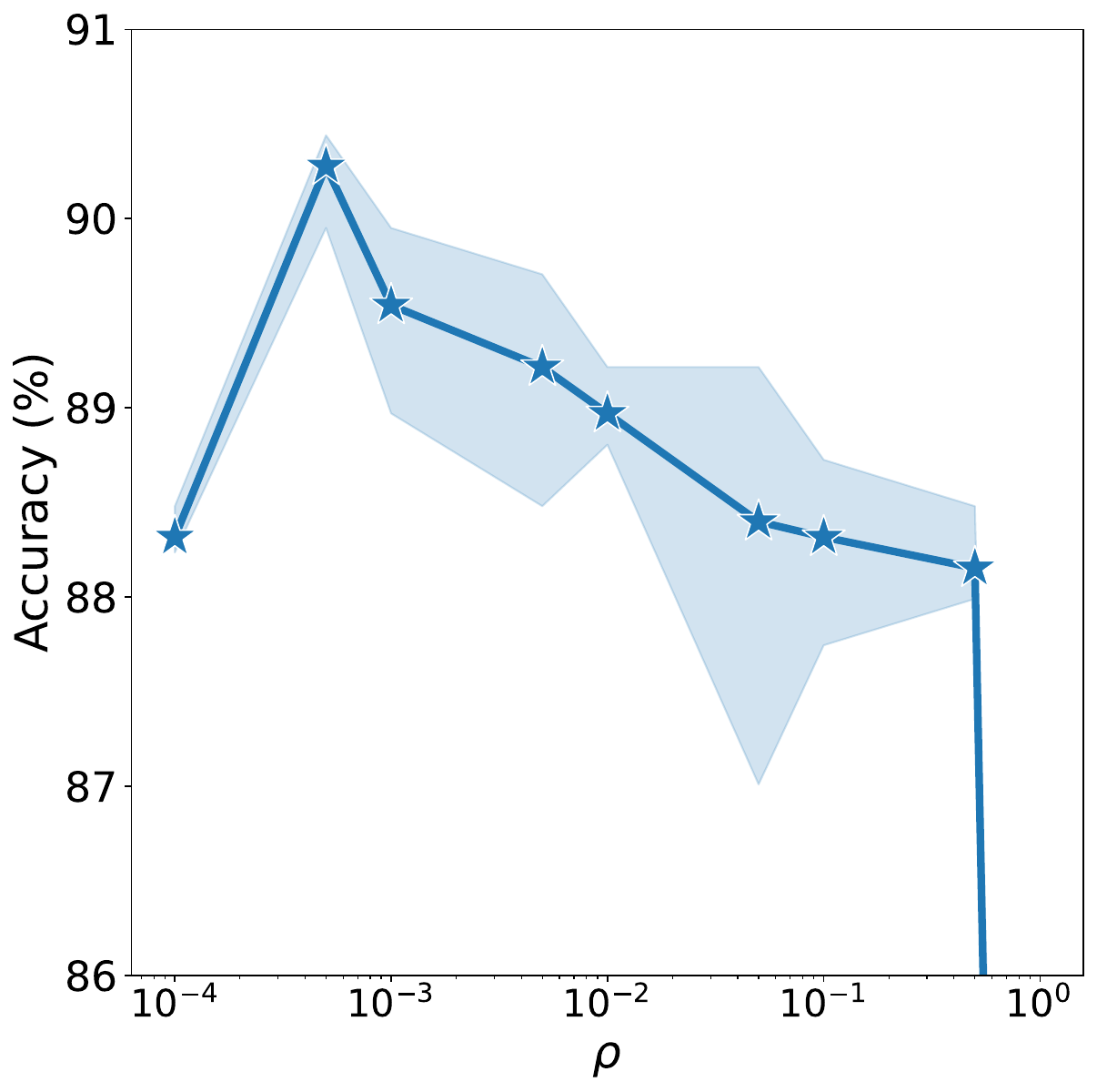}}
\vskip -.14in
\caption{Classification accuracy of SAM and ASAM with different $\rho$'s on \textit{MRPC} using \textit{DeBERTa}. As shown, the performance is sensitive to $\rho$.}
\label{fig:rho}
\vskip -.2in
\end{figure}

The perturbation radius $\rho$ controls the strength of penalizing sharp minima and plays an important role in SAM~\cite{foret2021sharpness,andriushchenko2022towards,kwon2021asam,zhuang2022surrogate}.
Figure \ref{fig:rho}
shows the classification accuracy of SAM and ASAM~\cite{kwon2021asam} w.r.t. different $\rho$'s on the \textit{MRPC} dataset using the \textit{DeBERTa} network, where the corresponding experimental setups are shown in Section \ref{sec:glue}.
As can be seen,	
SAM and ASAM prefer different $\rho$'s, and their performance is sensitive to $\rho$, emphasizing the crucial need for careful selection of $\rho$. 
To deal with this issue, recent attempts~\cite{foret2021sharpness,kwon2021asam,zhuang2022surrogate} perform grid search on $\rho$, which is straightforward but time-consuming.

To learn
the perturbation radius $\rho$ more efficiently, in this paper,
we propose a \underline{LE}arning the per\underline{T}urbation radiu\underline{S} (LETS) method 
by formulating the learning of $\rho$ as a bilevel optimization problem.
Specifically, in the lower-level problem,
a SAM model is obtained based on the training data and a given $\rho$, while in the upper-level problem, $\rho$ is updated by minimizing the gap between the validation and training losses based on the obtained SAM model, which is a function of $\rho$.
As the lower-level problem is usually nonconvex, 
we propose a gradient-based algorithm for updating model parameters and $\rho$
alternatively.
Experiments conducted on several benchmark datasets across diverse fields demonstrate that the proposed LETS is effective in learning a suitable radius.

In summary, our contributions are three-fold:
\begin{enumerate*}[(i), series = tobecont, itemjoin = \quad]
\item We formulate 
the problem of learning the perturbation radius as bilevel optimization and propose a 
gradient-based algorithm (called LETS-SAM) to adjust the radius for SAM.
\item We perform extensive experiments on various datasets  across computer vision and natural language process tasks as well as various network architectures across convolution-based and transformer-based networks
to verify that
LETS-SAM performs better than SAM.
\item LETS 
is general and can be combined with any SAM algorithm.
We integrate it into ASAM to propose LETS-ASAM.
Experimental results show
that LETS-ASAM achieves better performance than ASAM,
demonstrating the proposed LETS is effective.
\end{enumerate*}
	
\textbf{Notations.}	
Lowercase and uppercase boldface letters denote vectors (\eg, $\vx$) and matrices (\eg, $\vX$), respectively.
$\ell_2$-norm of $\vx$ is denoted by $\|\vx\|$.
$\operatorname{diag}(\vv)$ constructs a diagonal matrix with the vector $\vv$ on the diagonal.
For a vector $\vv\in \bR^d$,
$[\vv]^2\equiv [v_1^2, \dots, v_d^2]$ (resp. $|\vv|\equiv [|v_1|, \dots, |v_d|]$) denotes the elementwise square (resp. absolute) of $\vv$.
$\vI$ is the identity matrix. $\mathbf{1}_d$
is a $d$-dimensional all-ones vector. $\hD^{tr}=\{(\vx^{tr}_i, y^{tr}_i)\}_{i=1}^{N^{tr}}$, $\hD^{vl}=\{(\vx^{vl}_i, y^{vl}_i)\}_{i=1}^{N^{vl}}$ and $\hD^{ts}=\{(\vx^{ts}_i, y^{ts}_i)\}_{i=1}^{N^{ts}}$ represent the training, validation, and testing datasets, respectively. $f(\vx; \vtheta)$ denotes a model parameterized by $\vtheta$. $\hL(\hD;\vtheta)=\frac{1}{|\hD|}\sum_{(\vx_i, y_i) \in \hD}$
$\ell(f(\vx_i;\vtheta), y_i)$ denotes the loss on data set $\hD$ using model $\vtheta$,
where $\ell(\cdot, \cdot)$ denotes a loss function (e.g., cross-entropy loss for classification).
$\nabla \hL\left(\hD;\vtheta\right)$ and $\nabla^2 \hL\left(\hD;\vtheta\right)$ denote the gradient and Hessian of $\hL\left(\hD;\vtheta\right)$ w.r.t $\vtheta$, respectively.
	
%%%%%%%%%%%%%%%%%%%%%%%%
% related work section
%%%%%%%%%%%%%%%%%%%%%%%%
\section{Related Work}

\textbf{Generalization and Loss Landscape.}
As deep neural networks are powerful enough to memorize all training data,
seeking a model with better generalization ability is crucial to mitigate overfitting.
Recently, various works \cite{Jiang2020Fantastic,karolina2017,keskar2017on} conduct extensive experiments on various datasets to study the relationship between loss landscape and generalization, and found that
a flatter minima results in a better generalization. Therefore, several algorithms are proposed to improve the generalization ability of models by seeking flatter minima. For example, \cite{bisla2022low,zhou2019toward} add noise to model parameters, SWA and its variants~\cite{cha2021swad,izmailov2018averaging} average model parameters during training, 
\cite{zhao2022penalizing} penalizes gradient norm, and sharpness-aware minimization (SAM) as well as its variants~\cite{liu2022random,zhuang2022surrogate,foret2021sharpness,kwon2021asam} solves a min-max problem to search flat minima explicitly. 
These approaches have demonstrated superior results in various fields, including supervised learning~\cite{bisla2022low,qu2022generalized,zhao2022penalizing,foret2021sharpness,izmailov2018averaging}, 
transfer learning~\cite{foret2021sharpness,zhuang2022surrogate}, domain generalization~\cite{cha2021swad}, 
federated learning~\cite{qu2022generalized}, and natural language processing~\cite{bahri2022sharpness}. 

\textbf{Sharpness-Aware Minimization (SAM).}
SAM~\cite{foret2021sharpness}
seeks flat minima via
solving a min-max optimization problem as
\begin{align}
\min_{\vtheta} \max_{\| \vepsilon\| \leq \rho } \hL(\hD^{tr}; \vtheta + \vepsilon),
\label{eq:sam}
\end{align}
where $\rho > 0$ is a perturbation radius.
Intuitively, 
SAM aims to find a model $\vtheta$
such that all its neighbor models (within an $\ell_2$ ball of radius $\rho$) have low losses.
Due to the infeasible computation cost of solving the inner maximization problem for nonconvex losses, SAM approximates it via the first-order Taylor approximation and obtains the update rule at iteration $t$ as
\begin{align*}
\vtheta_{t+1} = \vtheta_{t} - \eta \nabla \hL\left(\hD^{tr}; \vtheta_t + \rho \frac{\nabla \hL(\hD^{tr}; \vtheta_t)}{\| \nabla \hL(\hD^{tr}; \vtheta_{t})\|}\right),
\end{align*}
where 
$\eta$ denotes the step size.
Note that SAM involves two gradient calculations at each iteration. To improve its efficiency, many methods have been proposed to reduce the computation cost of SAM.
For example,
Look-SAM \cite{liu2022towards} and RST~\cite{zhao2022ss} employ the SAM update periodically or randomly, respectively,
while 
AE-SAM \cite{jiang2023adaptive}
proposes 
an adaptive policy to
employ SAM 
only when the model is in sharp regions.
ESAM \cite{du2021efficient} proposes to perturb the chosen part of the model and only uses a selected subset of samples to compute the gradients.
To improve the effectiveness of SAM,
GSAM~\cite{zhuang2022surrogate}
proposes to minimize a surrogate gap $\max_{\| \vepsilon\| \leq \rho } \hL(\hD^{tr}; \vtheta + \vepsilon) - \hL(\hD^{tr}; \vtheta)$,
while 
RSAM~\cite{liu2022random}
simply injects Gaussian noises to perturb model parameters and ASAM~\cite{kwon2021asam} designs an adaptive sharpness measure by re-scaling.

For most of the SAM-based methods,
the perturbation radius $\rho$
is crucial
to their performance~\cite{foret2021sharpness,andriushchenko2022towards,kwon2021asam,zhuang2022surrogate}. 
Instead of performing a grid search over $\rho$
by cross-validation,
which is time-consuming, in this paper, we propose a gradient-based method to learn $\rho$.

\textbf{Bilevel Optimization.} 
Bilevel optimization
is first introduced in \cite{bracken1973mathematical}
and successfully used in a variety of areas,
for example, 
hyperparameter optimization \cite{liu2021value,feurer2019hyperparameter,Franceschi2018},
meta-learning \cite{Franceschi2018,jiang2021effective,jiang2022subspace,ye2021multiobjective},
prompt tuning~\cite{jiang2023effective},
and reinforcement learning \cite{stadie2020learning}.
Bilevel optimization consists of two problems:
a lower-level problem and an upper-level one.
The former acts as a constraint for the latter.
When the lower-level problem is convex,
one approach is to reformulate 
the bilevel problem as a single-level problem by replacing the lower-level problem with the first-order optimality condition~\cite{allende2013solving,sinha2019using}.
However,
in deep neural networks,
problems are usually nonconvex. 
Recently, gradient-based first-order methods~\cite{ghadimi2018approximation,hong2020two,liao2018reviving} for bilevel optimization have become popular due to their efficiency and effectiveness.
	
%%%%%%%%%%%%%%%%%%%%%%%%
% method section
%%%%%%%%%%%%%%%%%%%%%%%%
\section{The LETS Method}
In this section,
we formulate the objective function of the LETS method to learn the perturbation radius $\rho$
as bilevel optimization (i.e., Section \ref{sec: problem})
and 
propose a gradient-based algorithm to 
learn $\rho$ (i.e., Section \ref{sec: LETS-sam}).
Considering the generality of the proposed method, it can be combined with any SAM algorithm, and an example of integrating it into 
ASAM~\cite{kwon2021asam} is shown in Section \ref{sec:LETS-sam}. 
	
\subsection{Problem Formulation}\label{sec: problem}
	
We consider the SAM optimization in problem \eqref{eq:sam}.
Let $\vtheta^\star(\rho)$ be the solution,
which is a function of the perturbation radius $\rho$.
Though SAM 
has shown to be effective,
its generalization performance is sensitive to the choice of $\rho$ (i.e., Figure \ref{fig:rho}).
Instead of using grid search, which is simple but time-consuming,
we propose to learn $\rho$ in an end-to-end manner.
To achieve this,
we formulate the problem of learning $\rho$ into bilevel optimization, where the lower-level objective is the SAM problem and 
the upper-level objective is a generalization metric for $\vtheta^\star(\rho)$.
The choice of generalization metric is flexible, 
for example, the loss value on the validation set $\hL(\hD^{vl};\vtheta^\star(\rho))$, the generalization gap $\hL(\hD^{vl};\vtheta^\star(\rho)) - \hL(\hD^{tr};\vtheta^\star(\rho))$, or its square $\frac{1}{2}\left(\hL(\hD^{vl};\vtheta^\star(\rho)) - \hL(\hD^{tr};\vtheta^\star(\rho))\right)^2$.
Empirical results (i.e., Table~\ref{table:differen_matric_sam} in Section \ref{sec:abl}) show that the last is better and therefore is used.
Formally, the objective function of the proposed LETS method  
is formulated as
\begin{align}
    \underset{\rho \in (0, \infty)}{\min}&\quad\frac{1}{2}\left(\hL(\hD^{vl};\vtheta^\star(\rho)) - \hL(\hD^{tr};\vtheta^\star(\rho))\right)^2\label{eq:upper-level} \\ 
    \text{s.t.}&\quad\vtheta^\star (\rho)=\argmin_\vtheta \max_{\|\vepsilon \| \leq \rho} \hL(\hD^{tr};\vtheta+\vepsilon).
    \label{eq:lower-level}
\end{align}

\subsection{Learning Perturbation Radius for SAM}
\label{sec: LETS-sam}
	
When the lower-level problem 
is convex, 
one can seek the optimal solution 
$\vtheta^\star (\rho)$ by solving the low-level problem (\ref{eq:lower-level})
and update $\rho$
in the upper level by performing one gradient descent step,
where 
hyper-gradient $\nabla_{\rho}\frac12 (\hL(\hD^{vl};\vtheta^\star(\rho)
- \hL(\hD^{tr};\vtheta^\star(\rho))
)^2$
can be computed by iterative differentiation~\cite{pedregosa2016hyperparameter} or
approximate implicit differentiation~\cite{koh2017understanding}. 
However,
in deep neural networks, 
the lower-level problem is usually nonconvex,
thus,
seeking $\vtheta^\star (\rho)$
is computationally infeasible.
To address this problem,
we propose a gradient-based algorithm for updating the model parameters and $\rho$ alternatively.
The detailed procedure is shown in Algorithm \ref{alg:bsam}.

At iteration $t$, we sample
a batch $\hB_t^{tr}$ from the training dataset
and $\hB_t^{vl}$ from the validation dataset (i.e., steps \ref{alg-step:tr} and \ref{alg-step:vl}).
For the 
lower-level problem,
we take a gradient descent update (i.e., steps \ref{step:sam} and \ref{step:sam-2}) as
\begin{align}
\label{eq:theta_t}
\vtheta_{t+1} (\rho_{t}) =&\ \vtheta_{t} - \eta \nabla \hL\left(\hD^{tr};\vtheta_{t}+\rho_{t}\hat{\vepsilon}_{t}^{\text{(SAM)}}\right),
\end{align}
where $\hat{\vepsilon}_{t}^{\text{(SAM)}}\equiv \frac{\nabla \hL(\hD^{tr};\vtheta_{t})}{\|\nabla \hL(\hD^{tr};\vtheta_{t})\|}$ and  $\eta$ is the step size.
Here $\vtheta_{t+1} (\rho_{t})$ is an approximate solution to the SAM problem as we only conduct the gradient descent step once. Obviously $\vtheta_{t+1} (\rho_{t})$ is a function of $\rho_t$.

In the upper-level problem,
we perform a gradient descent step for updating $\rho$ (i.e., step \ref{step:rho-c}) as
\begin{align*}
\rho_{t+1} & \!= \rho_{t} \!-\! \beta \nabla_{\rho_{t}} \frac{1}{2}(\hL(\hD^{vl};\vtheta_{t+1} (\rho_{t})) - \hL(\hD^{tr};\vtheta_{t+1} (\rho_{t})))^2,
\end{align*}
where $\beta$ is the step size. $\nabla_{\rho_{t}} \frac{1}{2} (\hL(\hD^{vl};\vtheta_{t+1} (\rho_{t})) - \hL(\hD^{tr};\vtheta_{t+1} (\rho_{t})))^2$ is computed by the chain rule (i.e., steps \ref{step:rho-a} and \ref{step:rho-b}) as
$-  \eta  \nabla_{\vtheta_{t+1}}^\top\frac{1}{2}(\hL(\hD^{vl};\vtheta_{t+1}) 
- \hL(\hD^{tr};\vtheta_{t+1}))^2  \nabla^2 \! \hL\left(\hD^{tr};\vtheta_{t}\!+\!\rho_{t}\hat{\vepsilon}_{t}^{\text{(SAM)}}\right) \! \hat{\vepsilon}_{t}^{\text{(SAM)}}$. Details of the derivation are provided in Appendix A.
Here the first term of gradient is easy to compute as

\begin{align*}
\nabla_{\vtheta_{t+1}} \frac{1}{2}(\hL(\hD^{vl};\vtheta_{t+1}) \!-\! \hL(\hD^{tr};\vtheta_{t+1}))^2 &= (\hL(\hD^{vl};\vtheta_{t+1}) \!-\! \hL(\hD^{tr};\vtheta_{t+1}))  \\
&\left(\nabla_{\vtheta_{t+1}}\hL(\hD^{vl};\vtheta_{t+1}) \!-\!\nabla_{\vtheta_{t+1}} \hL(\hD^{tr};\vtheta_{t+1})\right)
\end{align*}
(i.e., steps \ref{step:sam-a} and \ref{step:sam-b}).
The second term needs to compute a Hessian, which is computationally expensive for large models like deep neural networks.
Following \cite{bottou2018optimization,khan2018fast}, 
the Hessian $\nabla^2 \hL \left(\hD^{tr};\vtheta_{t}+\rho_{t}\hat{\vepsilon}_t^{\text{(SAM)}}\right)$ can be approximated by a first-order derivative (i.e., step \ref{step: hessian}) as
\begin{align}
\operatorname{diag}\!\left(\!\left[ \! \nabla \hL\!\left(\hD^{tr};\vtheta_{t}\!+\!\rho_{t}\hat{\vepsilon}_t^{\text{(SAM)}}\right)\!\right]^2\right).\!\! \label{eq:hessian-approx}
\end{align}

As proved in Appendix B
, LETS-SAM has a convergence rate of $\mathcal{O}(\frac{1}{\sqrt{T}})$, which is the same as SAM~\cite{andriushchenko2022towards} and its variants~\cite{qu2022generalized,jiang2023adaptive} under similar conditions. The details of the theoretical analysis are provided in Appendix B.
Hence, adjusting the perturbation radius does not affect the convergence speed.
	
\begin{algorithm}[!t]
\caption{
    \underline{LE}arning per\underline{T}urbation radiu\underline{S} ({\color{darkred}LETS-SAM} and {\color{darkblue}LETS-ASAM}).}
\label{alg:bsam}
\begin{algorithmic}[1]
    \Require training set $\hD^{tr}$, validation set $\hD^{vl}$; 
    stepsizes $\beta$ and $\eta$, \#iterations $T$; model parameter $\vtheta$; initialization $\rho_0$ and $\vtheta_0$; {\color{darkblue}$\xi$ for LETS-ASAM};
    
    \For{$t=0,\dots, T-1$}
    \State sample a mini-batch training data $\hB_{t}^{tr}$ from $\hD^{tr}$; \label{alg-step:tr}
    \State sample a mini-batch validation data $\hB_{t}^{vl}$ from $\hD^{vl}$; \label{alg-step:vl}
    \State  $\vg_{t}^{tr}=\nabla \hL(\hB_{t}^{tr}; \vtheta_t)$;
    \State {\color{darkred}if LETS-SAM: $\hat{\vg}_{t}^{tr} = \nabla \hL\left(\hB_{t}^{tr}; \vtheta_t + \rho_{t} \frac{\vg_{t}^{tr}}{\|\vg_{t}^{tr}\|}\right)$;} \label{step:sam}
\State {\color{darkblue}if LETS-ASAM: $\hat{\vg}_{t}^{tr} \!=\! \nabla \hL\left(\hB_{t}^{tr}; \vtheta_t + \rho_{t} \frac{\vT_{\vtheta_{t}}^2 \vg_{t}^{tr}}{\|\vT_{\vtheta_{t}}\vg_{t}^{tr} \|}\right)$,}
    \Statex \;\;\;\; {\color{darkblue}where $\vT_{\vtheta_{t}}$ is computed by Eq. \eqref{T_thetax};\label{step:asam-1} }\label{step:asam}
    \State $\vtheta_{t+1} = \vtheta_t - \eta\hat{\vg}_{t}^{tr}$; \label{step:sam-2}
    \State  $\bar{\vg}_{t}^{tr} = \nabla \hL(\hB_{t}^{tr}; \vtheta_{t+1})$ and $\bar{\vg}_{t}^{vl} = \nabla \hL(\hB_{t}^{vl}; \vtheta_{t+1})$; \label{step:sam-a}
    \State $\vg_a = (\hL(\hB_{t}^{vl};\vtheta_{t+1}) - \hL(\hB_{t}^{tr}; \vtheta_{t+1})) (\bar{\vg}_{t}^{vl} - \bar{\vg}_{t}^{tr}) $; \label{step:sam-b}
    \State  $\vH = \operatorname{diag}([\hat{\vg}_{t}^{tr}]^2)$; \label{step: hessian}
    \State {\color{darkred}if LETS-SAM: $\vg_b = \vH \frac{\bar{\vg}_{t}^{tr}}{\|\bar{\vg}_{t}^{tr}\|}$;}
    \label{step:rho-a}
    \State {\color{darkblue}if LETS-ASAM: 
        $\vg_b = \vH \frac{\vT_{\vtheta_{t}}^2 \bar{\vg}_{t}^{tr}}{\|\vT_{\vtheta_{t}}\bar{\vg}_{t}^{tr}\|}$;} \label{step:asam-rho-a}
    \State $g_{\rho} = - \vg_a^\top \vg_b g_{\rho}$; \label{step:rho-b}
    \State $\rho_{t+1} = \rho_t - \beta \eta g_\rho$; \label{step:rho-c}
    \EndFor  \\
    \Return $\vtheta_T$.
\end{algorithmic}
\end{algorithm}

\subsection{LETS-ASAM}
\label{sec:LETS-sam}

As the proposed LETS method is very general and can be integrated into any variant of SAM,
we show an example by combining LETS with the recent state-of-the-art method ASAM~\cite{kwon2021asam}.
The combined algorithm called LETS-ASAM is shown in Algorithm \ref{alg:bsam}.

ASAM defines an
adaptive sharpness of the loss
function, whose maximization region is determined by the normalization operator.
Then, the objective function of ASAM is formulated as
\begin{align}
\vtheta^\star (\rho)\equiv\argmin_\vtheta \max_{\|\vT_{\vtheta}^{-1}\vepsilon \| \leq \rho} \hL(\hD^{tr};\vtheta+\vepsilon),
\label{eq:asam}
\end{align}
where $\vT_{\vtheta}^{-1}$
is a normalization operator at $\vtheta$.
For example, $\vT_{\vtheta}$ is defined as $\vT_{\vtheta} = \diag(|\vtheta|)+ \xi\vI$
for fully-connected layers, where
$\xi$ is a positive hyperparameter, and for convolutional layers, $\vT_{\vtheta}$ is defined as
\begin{align}
\vT_\vtheta = \operatorname{diag}\left(\left[\|\vc_1\|\mathbf{1}_{d_1}, ..., \|\vc_{k}\|\mathbf{1}_{d_k}, |\tilde{\vtheta}|\right]\right) + \xi\vI,
\label{T_thetax}
\end{align}
where $\vtheta = [\vc_1, \dots, \vc_k, \tilde{\vtheta}]$, $\vc_i$ is the flattened weight vector of the $i$th convolution filter with its dimension as $d_i$, and
$\tilde{\vtheta}$ denote parameters that are not contained in convolution filters.
To integrate LETS into ASAM,
we replace the lower-level problem \eqref{eq:lower-level} 
with problem \eqref{eq:asam}.
At iteration $t$,
the update rule at the lower level (i.e., steps \ref{step:asam-1} and \ref{step:sam-2}) becomes 
\begin{align}
\vtheta_{t+1} =&\ \vtheta_{t} - \eta \nabla \hL\left(\hD^{tr};\vtheta_{t}+\rho_{t}\hat{\vepsilon}^{\text{(ASAM)}}_t\right),
\end{align} 
where $\hat{\vepsilon}^{\text{(ASAM)}}_t \equiv \frac{\vT_{\vtheta_{t}}^2 \nabla \hL(\hD^{tr};\vtheta_{t})}{\|\vT_{\vtheta_{t}}\nabla \hL(\hD^{tr};\vtheta_{t})\|}$,
while the update rule at the upper level (i.e., steps \ref{step:asam-rho-a} and \ref{step:rho-b}) is 
\begin{align*}
\!\!\!\rho_{t+1} \!\!=\!\! \rho_{t} \!+\! \frac{\beta \eta}{2} \nabla_{\vtheta_{t+1}}^\top  \!\!\left(\hL(\hD^{vl}\!;\!\vtheta_{t+1}) \!-\! \hL(\hD^{tr}\!;\!\vtheta_{t+1})\right)^2 \!
\nabla^2 \!\hL\!\left(\!\!\hD^{tr}\!;\!\vtheta_t\!+\!\rho_{t}\hat{\vepsilon}^{\text{(ASAM)}}_t\!\!\right) \!\!\hat{\vepsilon}^{\text{(ASAM)}}_t, \!\!\! 
\end{align*}
where the Hessian can be approximately computed by using the first-order derivative as in Eq. \eqref{eq:hessian-approx}.

\section{Experiments}

In this section, 
we first compare the proposed LETS-SAM and LETS-ASAM methods with state-of-the-art SAM-based methods on computer vision tasks (e.g. \textit{CIFAR-10}, \textit{CIFAR-100}, and \textit{ImageNet}) and natural language processing (e.g., \textit{GLUE} and \textit{IWSLT’14 DE-EN}) tasks on various architectures.
Next, we evaluate the robustness of LETS-SAM and LETS-ASAM to label noise.
Furthermore, we conduct experiments to study the robustness of LETS to the initialization of $\rho$ (i.e., $\rho_0$) and the effects of different generalization metrics. 
To further illustrate the superior performance of LETS, we visualize the loss landscapes of models learned by the LETS methods. 
Finally, we empirically study the convergence of the proposed LETS method.

\textbf{Baselines.}
The proposed methods are compared with ERM, SAM~\cite{foret2021sharpness}, ESAM~\cite{du2021efficient}, RST~\cite{zhao2022ss}, LookSAM~\cite{liu2022towards}, AE-SAM~\cite{jiang2023adaptive}, AE-LookSAM~\cite{jiang2023adaptive}, ASAM~\cite{kwon2021asam}, and GSAM~\cite{zhuang2022surrogate}. ESAM selects some of the training samples to update the model and uses a subset of parameters to compute the perturbation. RST switches between SAM and ERM randomly for each iteration according to a Bernoulli trial with a probability 0.5. LookSAM uses SAM for every five steps. AE-SAM and AE-LookSAM use SAM adaptively. ASAM improves SAM by using an adaptive sharpness measure while GSAM improves SAM by minimizing a surrogate gap.
We use official implementations of those baselines.
      
\subsection{\textit{CIFAR-10} and \textit{CIFAR-100}}
	\label{sec:cifar}

\textbf{Setups.}
Experiments are conducted on the \textit{CIFAR-10} and \textit{CIFAR-100} datasets~\cite{krizhevsky2009learning}, 
each of which contains 50,000 images for training and 10,000 for testing.
We use four network architectures: \textit{ResNet-18}~\cite{he2016deep}, \textit{WideResNet-28-10}~\cite{zagoruyko2016wide}, \textit{PyramidNet-110}~\cite{han2017deep}, and \textit{ViT-S16}~\cite{dosovitskiy2021an}.
Following experimental setups in \cite{foret2021sharpness,kwon2021asam,jiang2023adaptive},
the batch size is set to $128$, and the SGD optimizer with momentum $0.9$ and weight decay $0.0005$ is used. In the SGD optimizer, for updating model parameters, an initial learning rate $0.1$ with the cosine learning rate scheduler is adopted, while we use an initial learning rate of $0.0001$ with an exponential learning rate scheduler to update $\rho$.
We train \textit{PyramidNet-100} for 300 epochs, \textit{ViT-S16} for 1200 epochs, and train \textit{ResNet-18} and \textit{WideResNet-28-10} for 200 epochs. 
As the \textit{CIFAR} datasets do not have a held-out validation set, 
following the practice introduced in 
\cite{liu2022auto_lambda},
mini-batches of validation data are randomly sampled from the training set. 
To ensure $\rho$ is positive, 
an exponential transformation $\exp(\cdot)$ is applied to $\rho$, i.e., $\rho=\exp(\nu)$, where $\nu$ is an unconstrained variable to be learned.
Experiments are repeated over three random seeds.
All implementation details are summarized in Appendix E.
	
\begin{table}[!t]
    \centering
    \caption{Classification accuracy (\%) on \textit{CIFAR-10} using various architectures. 
    The better result in each comparison
    group is \underline{underlined} and the best result across all the groups is in
    \textbf{bold}.
    }
    % \resizebox{0.46\textwidth}{!}{
    \begin{NiceTabular}{c c c c c}
    \CodeBefore  
    \rectanglecolor{Gray}{10-1}{10-5}
    \rectanglecolor{Gray}{12-1}{12-5}
    \Body
        \toprule
         & \textit{ResNet-18} & \textit{WideResNet-28-10} & \textit{PyramidNet-110} & \textit{ViT-S16} \\
        \arrayrulecolor{black!50}\specialrule{1.5pt}{.3\jot}{0.3pc}
        ERM & $95.41 \pm 0.03$  &  $96.34 \pm 0.12$ & $96.62 \pm 0.10$ &  $86.69 \pm 0.11$\\
        ESAM & $96.56 \pm 0.08$ & $97.29 \pm 0.11$ & $97.81 \pm 0.10$ & $84.27 \pm 0.11$\\
        RST & $96.40 \pm 0.16$ & $97.09 \pm 0.11$ & $97.22 \pm 0.10$ & $87.38 \pm 0.14$\\
        AE-SAM   & $96.63 \pm 0.04$ & $97.30 \pm 0.10$ & $97.90 \pm 0.09$ & $87.77 \pm 0.13$\\
        LookSAM & $96.32 \pm 0.12$ & $97.02 \pm 0.12$ & $97.10 \pm 0.11$ & $87.12 \pm 0.20$\\
        AE-LookSAM & $96.56 \pm 0.21$ & $97.15 \pm 0.08$ & $97.22 \pm 0.11$ & $87.32 \pm 0.11$\\
        GSAM & $96.61 \pm 0.05$ & $97.39 \pm 0.08$ & $97.65 \pm 0.05$ & $88.33 \pm 0.41$\\
        \cmidrule{1-5} 
        SAM & $96.52 \pm 0.12$ & $97.27 \pm 0.11$  & $97.30 \pm 0.10$ & $87.37 \pm 0.09$\\
        LETS-SAM & $\mathbf{\underline{96.81}}\pm 0.02$ & $\underline{97.49} \pm 0.08 $  & $\underline{97.79} \pm 0.06$ & $\underline{88.83} \pm 0.17$\\
        \cmidrule{1-5}
        ASAM & $96.57 \pm 0.02$ & $97.28 \pm 0.07$  & $97.58 \pm 0.06$  & $90.35 \pm 0.05$\\
        LETS-ASAM & $\underline{96.77} \pm 0.01$ & $\mathbf{\underline{97.54}} \pm 0.08$ & $\mathbf{\underline{97.91}} \pm 0.01$ & $\mathbf{\underline{90.75} \pm 0.37}$\\
        \arrayrulecolor{black!50}\specialrule{1.5pt}{.3\jot}{0.3pc}
    \end{NiceTabular}
    \label{table:result-cifar10} 
\end{table}

\begin{table}[!ht]
    \centering
    \caption{Classification accuracy (\%) on \textit{CIFAR-100} using various architectures. 
    The better result in each comparison
    group is \underline{underlined} and the best result across all the groups is in
    \textbf{bold}.
    }
    \begin{NiceTabular}{c c c c c}
    \CodeBefore  
    \rectanglecolor{Gray}{10-1}{10-5}
    \rectanglecolor{Gray}{12-1}{12-5}
    \Body
        \toprule
         & \textit{ResNet-18} & \textit{WideResNet-28-10} & \textit{PyramidNet-110} & \textit{ViT-S16} \\
        \arrayrulecolor{black!50}\specialrule{1.5pt}{.3\jot}{0.3pc}
        ERM  & $78.17 \pm 0.05$ & $81.56 \pm 0.14$ & $81.89 \pm 0.15$ & $62.42 \pm 0.22$\\
         ESAM & $80.41 \pm 0.10$ & $84.51 \pm 0.02$ & $85.39 \pm 0.05$ & $62.11 \pm 0.15$\\
         RST  & $80.10\pm0.16$ & $82.89 \pm 0.02$ & $84.90 \pm 0.05$ &  $63.18 \pm 0.19$\\
         AE-SAM & $80.48 \pm 0.11$  & $84.51 \pm 0.11$ & $85.58\pm 0.10$ & $63.68 \pm 0.23$\\
         LookSAM & $79.89 \pm 0.29$ & $83.70\pm0.12$ & $84.01 \pm 0.06$ & $63.52 \pm 0.19$\\
         AE-LookSAM & $80.29 \pm 0.37$ & $83.92\pm0.07$ & $84.80 \pm 0.13$ & $64.16 \pm 0.23$\\
         GSAM  & $80.27 \pm 0.33$ & $83.80 \pm 0.08$ & $84.91 \pm 0.29$ & $63.21 \pm 0.38$\\
        \cmidrule{1-5} 
         SAM  & $80.17 \pm 0.15$ & $83.42 \pm 0.05$ & $84.46 \pm 0.05$ & $63.23 \pm 0.25$\\
         LETS-SAM  & $\underline{80.71} \pm 0.07$ & $\mathbf{\underline{84.78}} \pm 0.27$ & $\mathbf{\underline{85.86} \pm 0.23}$ & $\underline{64.66} \pm 0.46$\\
        \cmidrule{1-5}
         ASAM  & $80.66 \pm 0.16$ & $83.68 \pm 0.12$ & $85.13 \pm 0.12$  & $66.44 \pm 0.26$\\
        LETS-ASAM & $\mathbf{\underline{81.42}} \pm 0.07$ & $\underline{84.73} \pm 0.05$ & $\underline{85.47} \pm 0.10$ & $\mathbf{\underline{66.64} \pm 0.43}$\\
        \arrayrulecolor{black!50}\specialrule{1.5pt}{.3\jot}{0.3pc}
    \end{NiceTabular}
    \label{table:result-cifar100} 
\end{table}
	
% \noindent
\textbf{Results.}
The experimental results on \textit{CIFAR-10} and \textit{CIFAR-100} are shown in Table \ref{table:result-cifar10} and \ref{table:result-cifar100}, respectively.
We can find that, 
by learning the perturbation radius, 
LETS-SAM performs better than SAM on
all the four architectures.
Compared with ASAM, 
LETS-ASAM is also better,
demonstrating the effectiveness of the proposed LETS method.
Furthermore, on \textit{CIFAR-100},
LETS-ASAM achieves the highest accuracy on \textit{ResNet-18} and \textit{ViT-S16},
while LETS-SAM is the best on 
\textit{WideResNet-28-10} and \textit{PyramidNet-110}. On \textit{CIFAR-10},  LETS-SAM achieves the highest accuracy on \textit{ResNet-18},
while LETS-ASAM outperforms all the baseline models on 
\textit{WideResNet-28-10}, \textit{PyramidNet-110}, and \textit{ViT-S16}.
    
Figure \ref{fig:generalization-gap-cifar100} (resp. Figure 7
in Appendix D.3)
shows the generalization gap (i.e., $\hL(\hD^{ts}; \vtheta_t) - \hL(\hD^{tr}; \vtheta_t)$) w.r.t. training epochs on \textit{CIFAR-100} (resp. \textit{CIFAR-10}) dataset. 
As shown,
LETS-SAM (resp. LETS-ASAM) has a smaller generalization gap than SAM (resp. ASAM) when the training process nearly converges,
verifying that learning the perturbation radius can reduce the generalization gap.

\begin{figure}[!t]
    \centering
    \subfigure[\textit{ResNet-18}. \label{fig:cifar-100-resnet10}]{\includegraphics[width=0.3\textwidth]{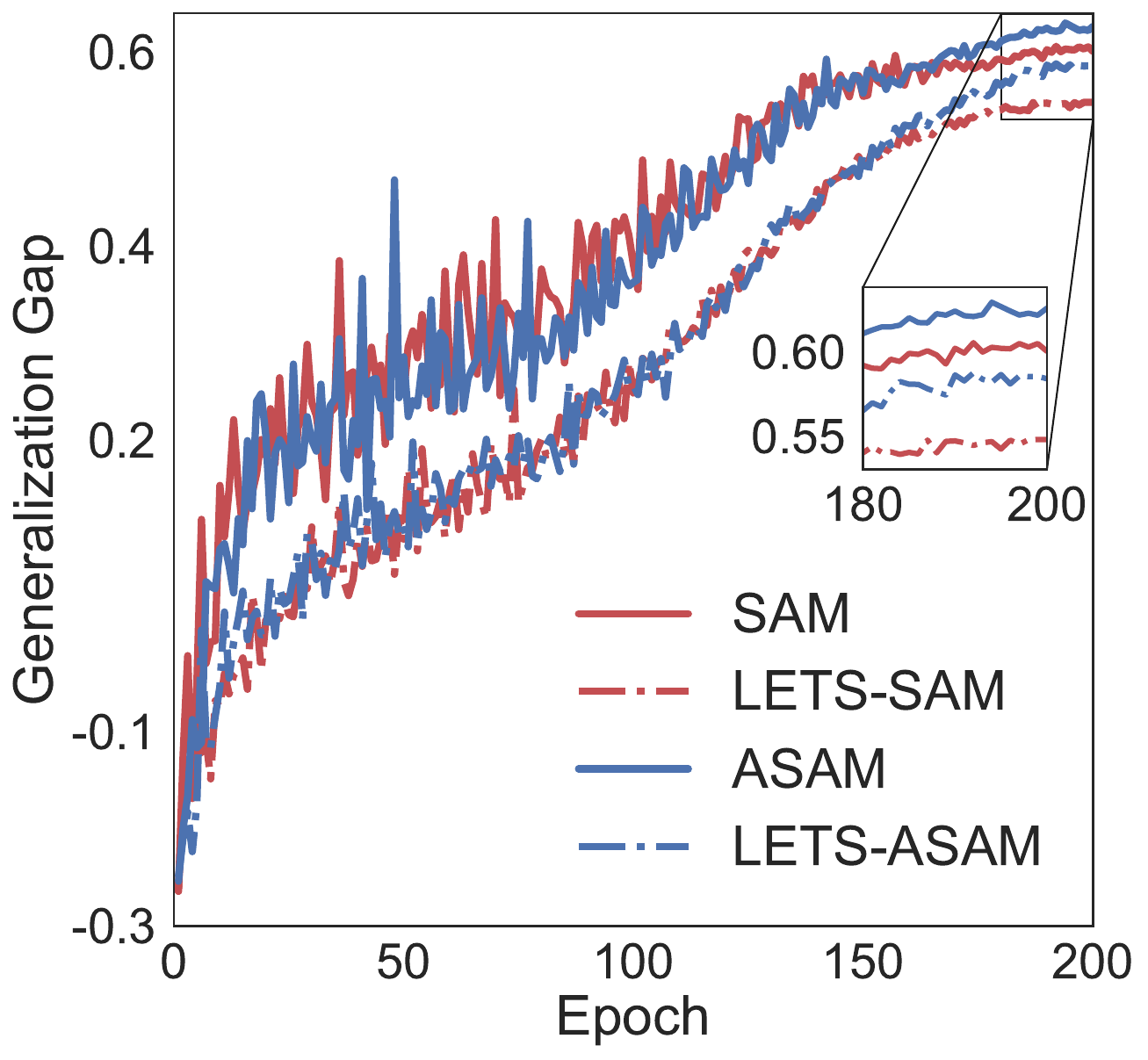}}
    \quad
    \subfigure[\!\textit{WideResNet-28-10}. \label{fig:cifar-100-wrn}]{\includegraphics[width=0.3\textwidth]{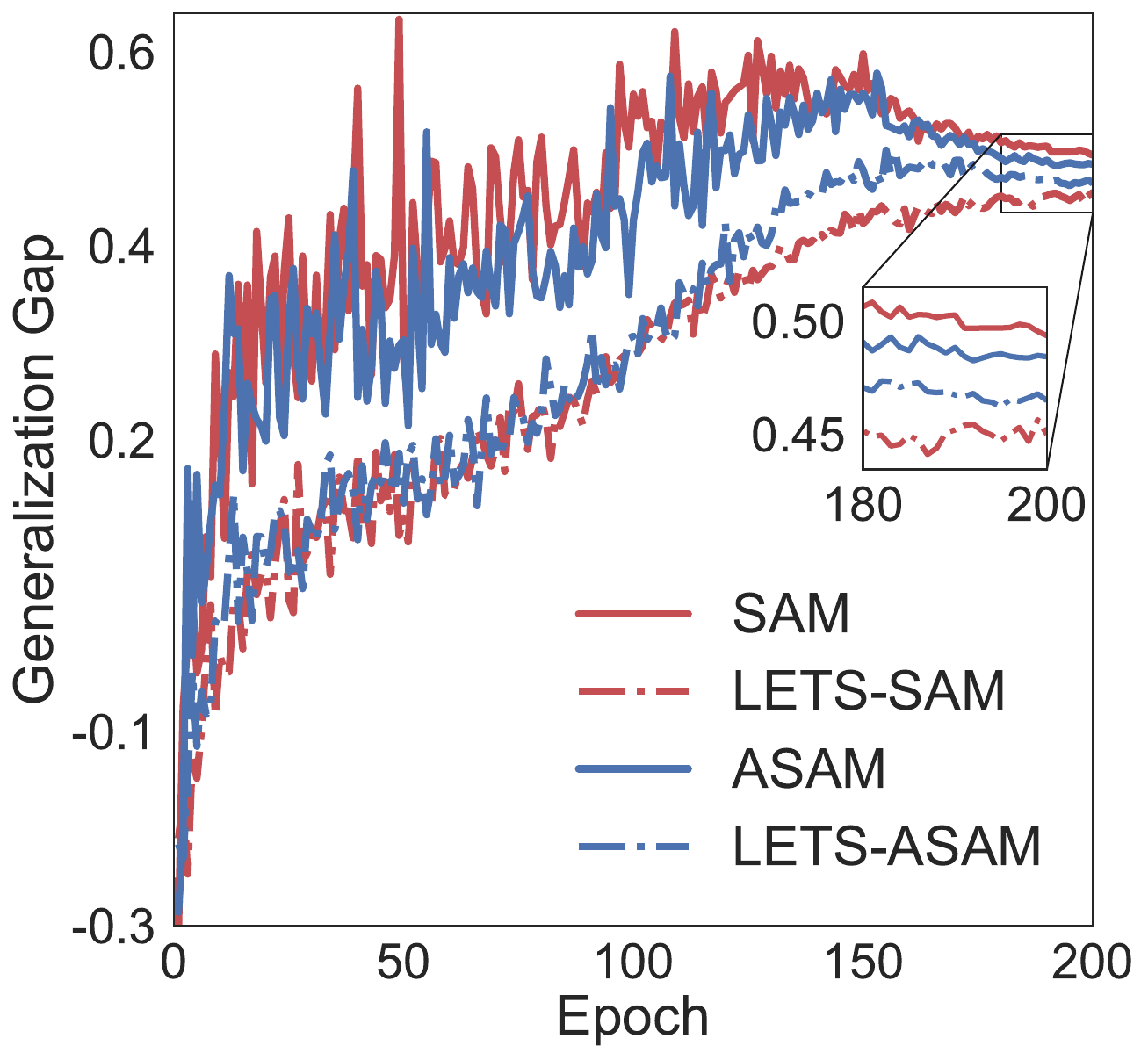}}
    \quad
    \subfigure[\textit{PyramidNet-110}. \label{fig:cifar-100-pym}]{\includegraphics[width=0.3\textwidth]{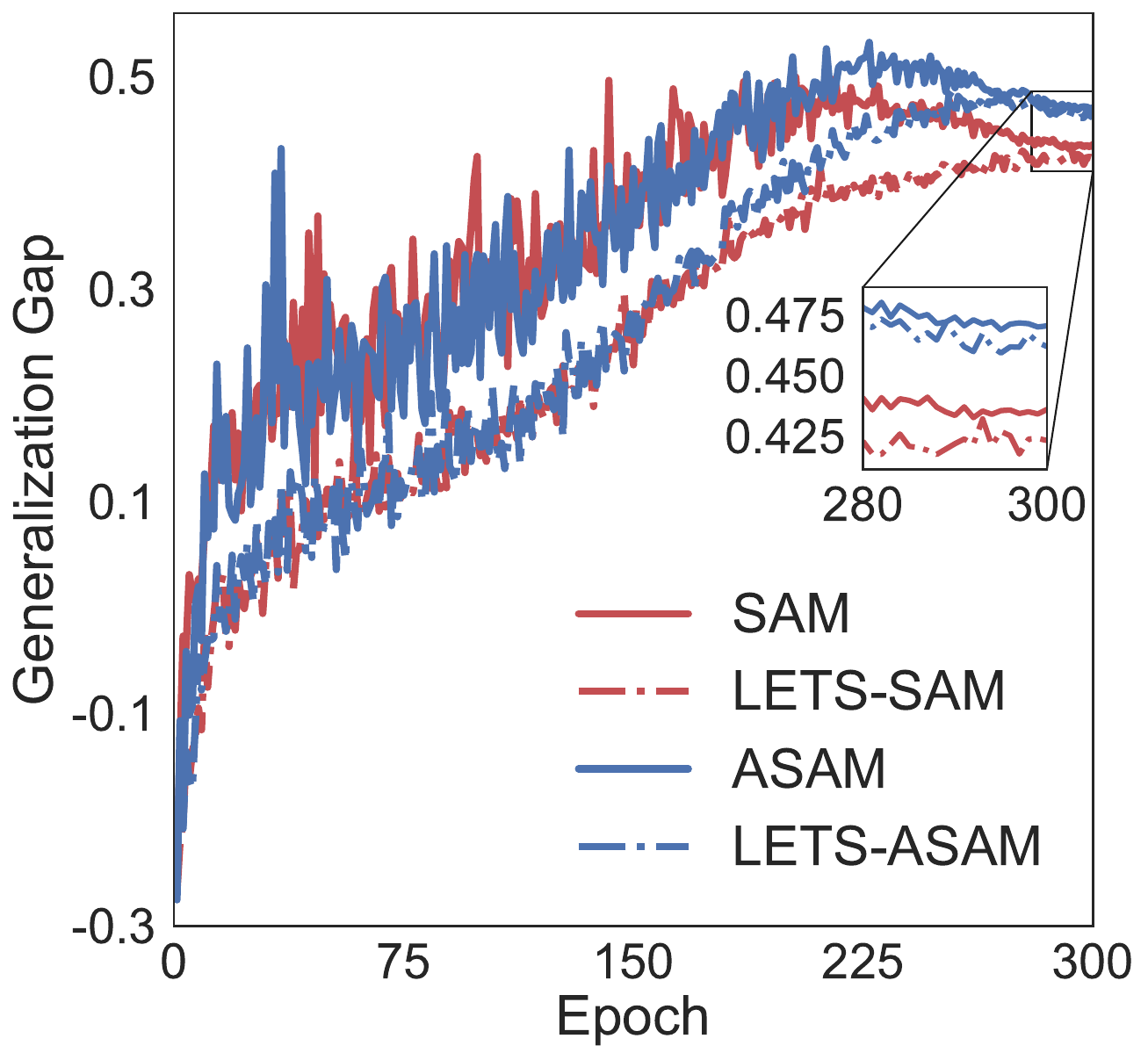}}
    \!\!\!
    \caption{Generalization gap
        \textit{w.r.t.} training epochs
        on \textit{CIFAR-100}. 
        Best viewed in color.
    }
    \label{fig:generalization-gap-cifar100}
\end{figure}

\subsection{\textit{ImageNet}}

% \noindent
\textbf{Setups.}
In this section, we conduct experiments on the \textit{ImageNet} dataset~\cite{russakovsky2015imagenet}, which contains $1,281,167$ images for training and $32,702$ images for testing, by using \textit{ResNet-50}~\cite{he2016deep}. 
Following the experimental setup in \cite{du2021efficient,jiang2023adaptive},
we use a batch size of $512$, SGD optimizer with momentum $0.9$, weight decay $0.0001$, an initial learning rate $0.1$ with the cosine learning rate scheduler for model parameters, and an initial learning rate $0.0001$ with the exponential learning rate scheduler for $\rho$.
The number of training epochs is $90$. Mini-batches of validation data are randomly sampled from the training set as in Section \ref{sec:cifar}.
Experiments are repeated over three random seeds.

% \noindent
\textbf{Results.}
The experimental results on \textit{ImageNet} are shown in Table \ref{tab:ImageNet}.
We can find that
LETS-SAM performs better than all the baseline methods.
Compared with ASAM, LETS-ASAM achieves a higher accuracy, demonstrating the effectiveness of the proposed LETS.
	
\begin{table}[!t]
    \centering
    \caption{Classification accuracy (\%) on the \textit{ImageNet} dataset. 
        The better result in each comparison
    group is \underline{underlined} and the best result across all the groups is in
    \textbf{bold}.}
    \begin{tabular}{cc}
        \toprule
        ERM &  $77.11 \pm 0.14$\\
        % \midrule
        ESAM & $77.25\pm 0.75$ \\
        RST  & $77.38\pm 0.06$\\
        AE-SAM & $77.43\pm 0.06$\\
        LookSAM & $77.13\pm 0.09$\\
        AE-LookSAM & $77.29\pm 0.08$\\
        GSAM & $77.20 \pm 0.00$\\
        \midrule
        SAM &  $77.47 \pm 0.12$\\ \rowcolor{Gray}
        LETS-SAM & $\mathbf{\underline{77.67}} \pm0.11$\\
        \midrule
        ASAM & $77.17\pm0.15$\\ \rowcolor{Gray}
        LETS-ASAM & $\underline{77.61} \pm 0.10$\\
        \bottomrule
    \end{tabular}
    \label{tab:ImageNet}
\end{table}

 \subsection{\textit{IWSLT’14 DE-EN}}

\begin{figure}
\centering
\includegraphics[width=0.5\textwidth]{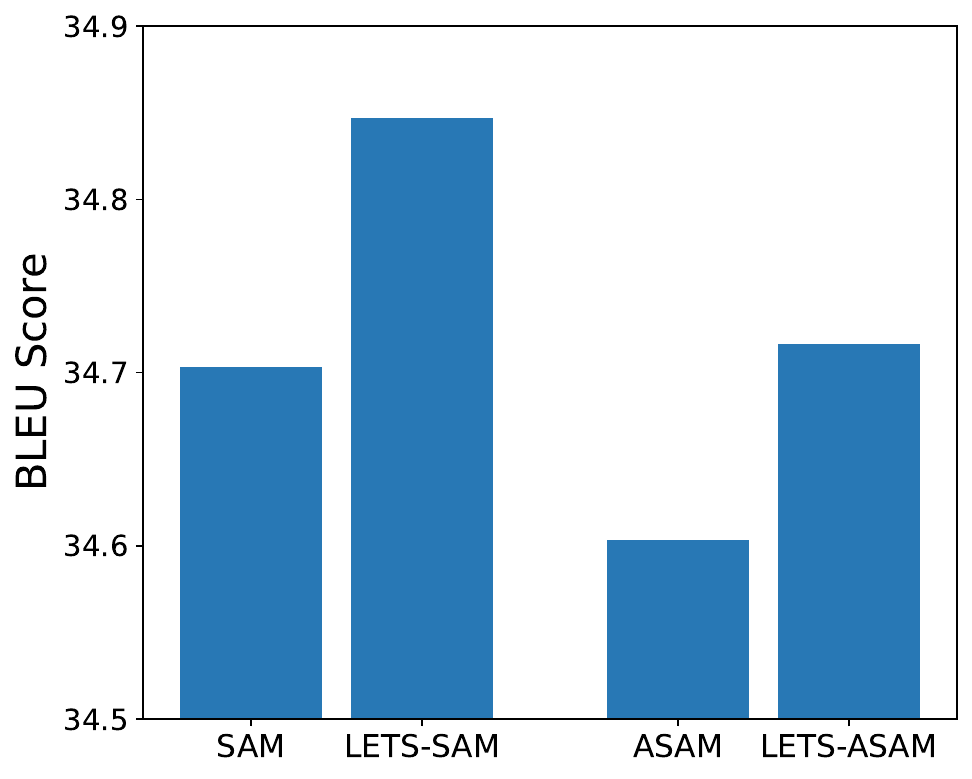}
\caption{Experimental results on \textit{IWSLT’14 DE-EN}.}
\label{fig:iwslt}
\end{figure}

\textbf{Setups.}
In this section, we conduct experiments on the \textit{IWSLT’14 DE-EN} dataset, which is a widely used dataset for machine translation. 
Following experimental setups in \cite{kwon2021asam}, we use the widely used machine translation architecture: Transformer architecture~\cite{vaswani2017attention}. We use the Adam optimizer with $(\beta_1,\beta_2)=(0.9,0.98)$ and weight decay 0.0001, initial learning rate 0.0005 for model parameters, initial learning rate 0.0001 with the exponential learning rate scheduler for $\rho$, and a dropout rate 0.3. Label smoothing is adopted with a factor of 0.1. The number of training epochs is $50$. Mini-batches of validation data are randomly sampled from the training set as in Section \ref{sec:cifar}. Following \cite{kwon2021asam}, 
we use the BLEU score as the evaluation metric (higher is better). Experiments are repeated over three random seeds.

% \noindent
\textbf{Results.}
Experimental results on the \textit{IWSLT'14 DE-EN} dataset are shown in Figure \ref{fig:iwslt}.
We can find that
LETS-SAM performs better than SAM and achieves the best performance, while 
LETS-ASAM also outperforms ASAM,
demonstrating the effectiveness of the proposed LETS method.

\subsection{\textit{GLUE}}
\label{sec:glue}

\textbf{Setups.} In this section, we conduct experiments on the \textit{GLUE} benchmark~\cite{wang2018glue}, which has various corpora and natural language understanding (NLU) tasks.
Each task has respective corpora and metrics. The details of the \textit{GLUE} benchmark are summarized in Appendix C.
We fine-tune the pre-trained checkpoint of the DeBERTa-base model on the \textit{GLUE} benchmark. Following the experimental setups in \cite{he2020deberta}, we use Adam optimizer with $\epsilon=10^{-6}$ and $(\beta_1,\beta_2)=(0.9, 0.999)$, linear learning rate scheduler with warmup steps and gradient clipping 1.0. Mini-batches of validation data are randomly sampled from the training set as in Section \ref{sec:cifar}. Experiments are repeated over three random seeds.

% \noindent
\textbf{Results.}
The experimental results on five NLU tasks of \textit{GLUE} are shown in Figure \ref{fig:glue}.
We can find that LETS-SAM performs better than SAM as shown in Figure \ref{fig:glue-sam}. Compared with ASAM, LETS-ASAM achieves better performance shown in Figure \ref{fig:glue-asam}, demonstrating the effectiveness of the proposed LETS method.
Due to page limit, the overall results on the \textit{GLUE} benchmark are reported in Table 9
of Appendix D,
which shows the superiority of the proposed LETS method.

\begin{figure}[!t]
    \centering
    \!\!\!
    \subfigure[SAM and LETS-SAM. \label{fig:glue-sam}]{\includegraphics[width=0.49\textwidth]{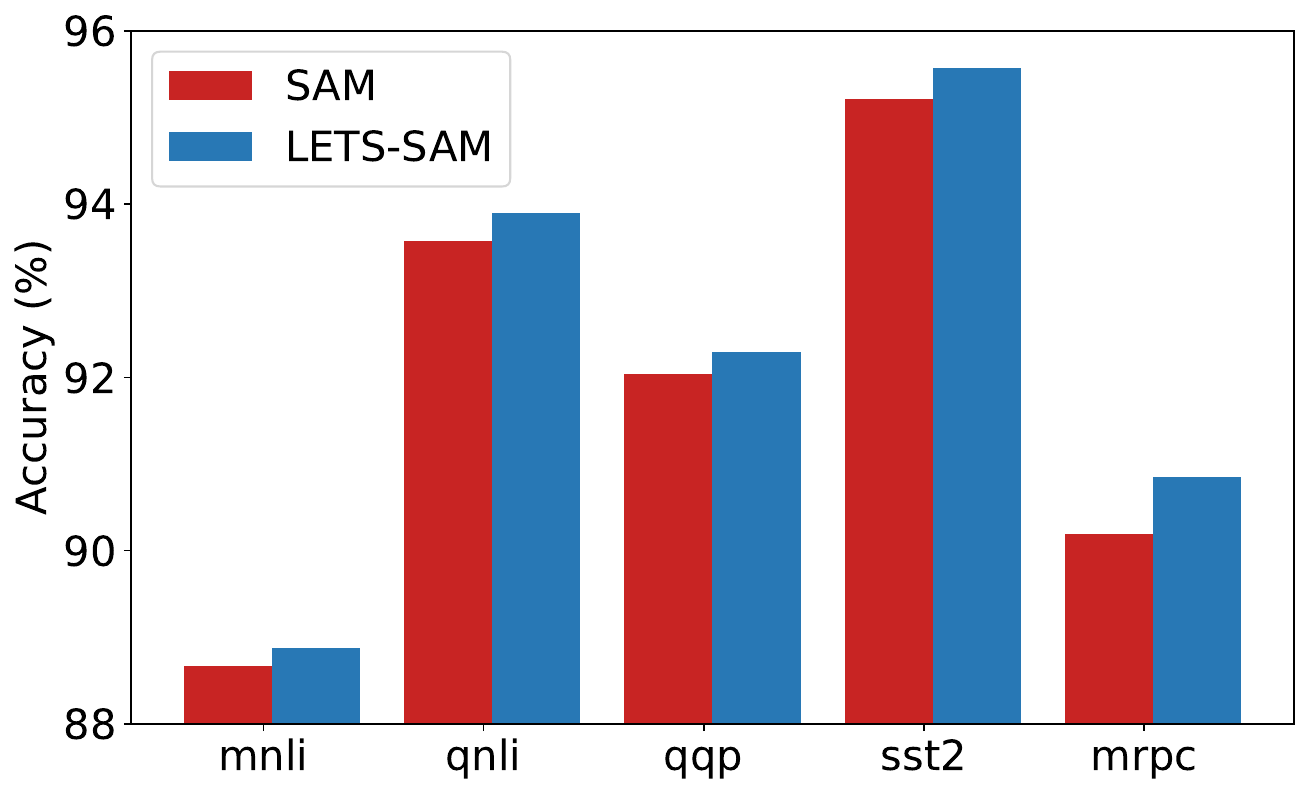}}
    \subfigure[ASAM and LETS-ASAM. \label{fig:glue-asam}]{\includegraphics[width=0.49\textwidth]{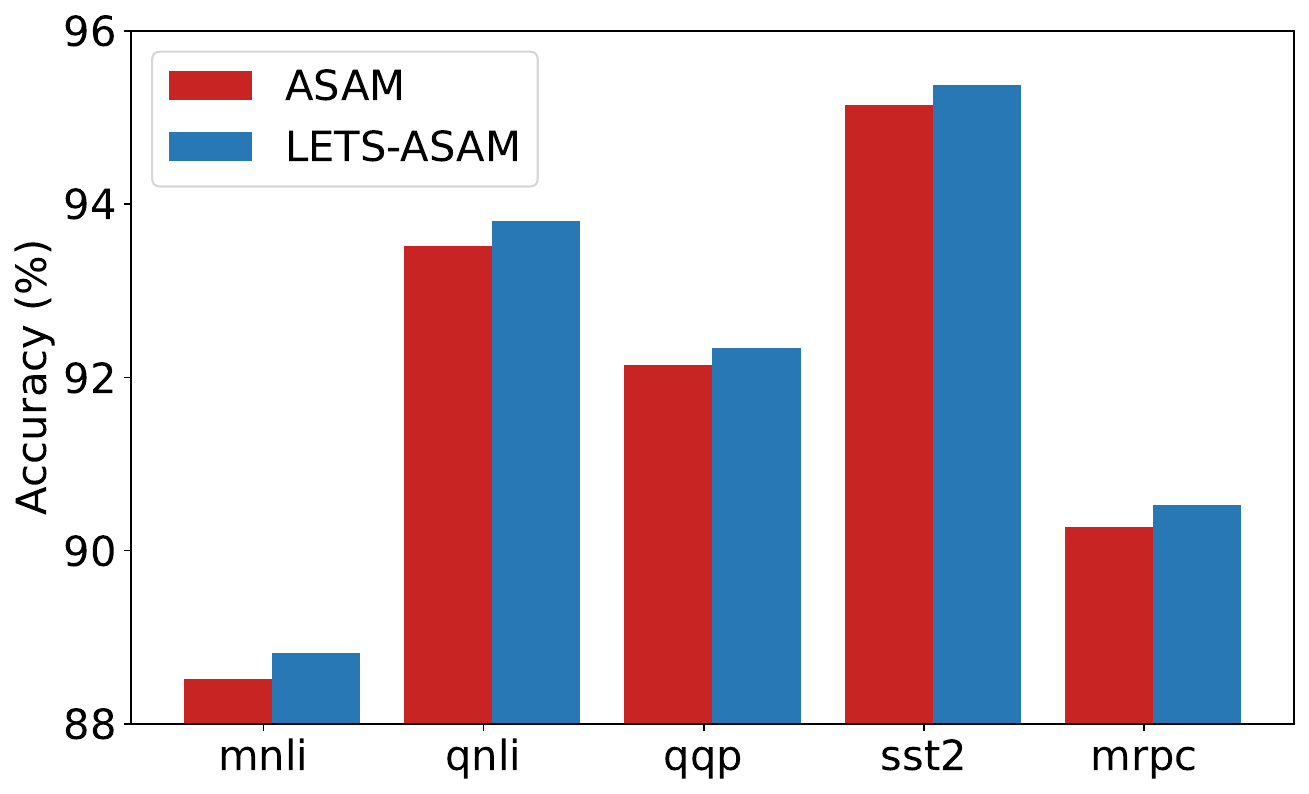}}
    \!\!\!
    \caption{Testing accuracy on five datasets from \textit{GLUE}.
    }
    \label{fig:glue}
\end{figure}

\subsection{Robustness to Label Noise}
	
% \noindent
\textbf{Setups.}
SAM has shown to be robust to label noise in training data \cite{foret2021sharpness}. 
In this section, we follow the experimental setups in \cite{foret2021sharpness,jiang2023adaptive} to study whether learning the perturbation radius can enhance the robustness of SAM.
The \textit{ResNet-18} and \textit{ResNet-32} are used.
We train the model on a corrupted version of the \textit{CIFAR-10} dataset (with noise levels of $20\%$, $40\%$, $60\%$, and $80\%$),
where the labels of some training data are flipped randomly while the testing set is kept clean.
We use batch size $128$,
SGD optimizer with momentum $0.9$ and weight decay $0.0005$,
initial learning rate $0.1$ with the cosine learning rate scheduler
for model parameters, and initial learning rate $0.0001$ with the exponential learning rate scheduler for $\rho$.
Mini-batches of validation data are randomly sampled from the training set as in Section \ref{sec:cifar}.
The number of training epochs is set to $200$.
Each experiment is repeated over three random seeds.

% \noindent
\textbf{Results.}
The results on \textit{ResNet-18} and \textit{ResNet-32} are shown in Table \ref{tab:noise_label_with_std}.
We can find that
LETS-SAM performs the best in all the noise levels.
Moreover, LETS-ASAM outperforms ASAM by a large margin.
Those results confirm that LETS is an effective method to improve the robustness of SAM and ASAM.

 \begin{table}[!tb]
    \centering
    \caption{Classification accuracy (\%) on \textit{CIFAR-10} for \textit{ResNet-18} and \textit{ResNet-32} trained with different levels of label noises. 
        The better result in each comparison
    group is \underline{underlined} and the best result across all the groups is in
    \textbf{bold}. }
    % \vskip -.08in
    % \resizebox{0.95\textwidth}{!}{
        % \begin{tabular}{cccccc}
        \begin{NiceTabular}{cccccc}
         \CodeBefore  
        \rectanglecolor{Gray}{10-2}{10-6}
        \rectanglecolor{Gray}{12-2}{12-6}
        \rectanglecolor{Gray}{21-2}{21-6}
        \rectanglecolor{Gray}{23-2}{23-6}
        \Body
            \toprule
            & & noise = 20\% & noise = 40\% & noise = 60\% & noise = 80\% \\
            \arrayrulecolor{black!50}\specialrule{1.5pt}{.3\jot}{0.3pc}
          \multirow{11}{*}{\STAB{\rotatebox[origin=c]{90}{\textit{ResNet-18}}}} 
            & ERM & $87.92 \pm 0.02$ & $70.82 \pm 0.33$ & $49.61 \pm 0.39$ & $28.23 \pm 0.40$\\
            & ESAM & $94.19 \pm 0.10$ & $91.46 \pm 0.49$ & $81.30 \pm 0.69$ & $15.00 \pm 4.89$\\
            & RST & $90.62 \pm 0.37$ & $77.84 \pm 0.56$ & $61.18 \pm 0.87$ & $47.32 \pm 1.50$\\
            & AE-SAM & $92.84 \pm 0.25$ & $84.17 \pm 0.53$ & $73.54 \pm 0.50$ & $65.00 \pm 2.25$\\
            & LookSAM & $92.72 \pm 0.18$ & $88.04 \pm 0.40$ & $72.26 \pm 1.75$ & $69.72 \pm 1.52$\\
            & AE-LookSAM & $94.34 \pm 0.29$ & $91.58 \pm 0.54$ & $87.85 \pm 0.23$ & $76.90 \pm 0.32$\\
            & GSAM & $91.72 \pm 0.15$ & $87.88 \pm 0.50$ & $83.29 \pm 0.25$ & $73.16 \pm 1.65$  \\
            \cmidrule{2-6} 
            & SAM & $94.80 \pm 0.05$ & $91.50 \pm 0.22$ & $88.15 \pm 0.23$ & $77.40 \pm 0.21$\\
            & LETS-SAM & $\mathbf{\underline{95.65}} \pm 0.09$ &  $\mathbf{\underline{93.84}} \pm 0.19$ & $\mathbf{\underline{89.48}} \pm 0.31$ & $\mathbf{\underline{77.89}} \pm 0.80$\\
            \cmidrule{2-6} 
            & ASAM & $91.47 \pm 0.21$ & $88.28 \pm 0.16$ & $83.22 \pm 0.38$ & $71.77 \pm 1.41$\\
            & LETS-ASAM & $\underline{92.77} \pm 0.18$ & $\underline{89.72} \pm 0.20$ & $\underline{84.94} \pm 0.16$ & $\underline{75.00} \pm 0.56$ \\
            \arrayrulecolor{black!50}\specialrule{1.5pt}{.3\jot}{0.3pc}
          \multirow{11}{*}{\STAB{\rotatebox[origin=c]{90}{\textit{ResNet-32}}}} 
            & ERM & $87.43 \pm 0.00$ & $70.82 \pm 0.98$ & $46.26 \pm 0.18$ & $29.00 \pm 1.79$\\
            & ESAM & $93.42 \pm 0.50$ & $91.63 \pm 0.29$ & $82.73 \pm 1.21$ & $10.09 \pm 0.10$\\
            & RST & $89.63 \pm 0.26$ & $74.17 \pm 0.47$ & $58.40 \pm 2.95$ & $59.53 \pm 1.63$\\
            & AE-SAM & $92.87 \pm 0.17$ & $82.85 \pm 2.16$ & $71.50 \pm 0.74$ & $65.43 \pm 3.19$\\
            & LookSAM & $92.49 \pm 0.05$ & $86.56 \pm 0.92$ & $63.35 \pm 0.48$ & $68.01 \pm 5.37$\\
            & AE-LookSAM & $94.70 \pm 0.10$ & $91.80 \pm 0.87$ & $88.22 \pm 0.27$ & $77.03 \pm 0.16$\\
            & GSAM & $92.07 \pm 0.13$ & $80.61 \pm 0.45$ & $84.08 \pm 0.47$ & $72.46 \pm 1.85$ \\
            \cmidrule{2-6} 
            & SAM & $95.08 \pm 0.23$ & $91.01 \pm 0.41$ & $88.90 \pm 0.39$ & $77.32 \pm 0.12$\\
            & LETS-SAM & $\mathbf{\underline{95.73}} \pm 0.10$ & $\mathbf{\underline{93.96}} \pm 0.05$ & $\mathbf{\underline{89.71}} \pm 0.17$ & $\mathbf{\underline{77.39}} \pm 0.19$ \\
            \cmidrule{2-6} 
            & ASAM & $91.61 \pm 0.26$ & $88.83 \pm 0.76$ & $83.61 \pm 0.33$ & $72.32 \pm 1.15$ \\
            & LETS-ASAM & $\underline{92.80} \pm 0.16$ & $\underline{89.91} \pm 0.41$ & $\underline{85.29} \pm 0.38$ & $\underline{75.55} \pm 1.06$\\
            \bottomrule
        \end{NiceTabular}
    \label{tab:noise_label_with_std}
\end{table}
	
\subsection{Robustness to the Initialization of $\rho$}
In this section, we conduct experiments on the \textit{CIFAR-10} and \textit{CIFAR-100} datasets using \textit{ResNet-18} to study the effect of the initialization of $\rho$ (i.e., $\rho_0$) to the performance of LETS-ASAM.
According to results shown in Table \ref{table:different_init_rho}, we can find that the performance of LETS-ASAM is not so sensitive to a wide range of $\rho_0 \in \{0.01, 0.05, 0.1, 0.5, 1, 1.5, 2\}$.
Hence, $\rho_0$ can be initialized more randomly without compromising the performance of LETS-ASAM, which could imply that learning the perturbation radius is more efficient and effective than using grid search to find the perturbation radius.

 \begin{table}[!t]
\centering
\caption{Classification accuracy (\%) of LETS-ASAM on \textit{CIFAR-10} and \textit{CIFAR-100} for different initializations of $\rho$.}
\begin{tabular}{ccc}
\toprule
$\rho_0$ &  {\textit{CIFAR-10}} & {\textit{CIFAR-100}}\\
\midrule
$0.01$ & $96.79 \pm 0.09$ & $81.37 \pm 0.18$\\
$0.05$ & $96.73 \pm 0.10$ & $81.62 \pm 0.14$\\
$0.1$ &  $96.78 \pm 0.08$ & $81.72 \pm 0.07$\\
$0.5$ & $96.74 \pm 0.03$ & $81.51 \pm 0.14$\\
$1$ & $96.77 \pm 0.01$ & $81.36 \pm 0.21$\\
$1.5$ & $96.79 \pm 0.03$ & $81.65 \pm 0.06$\\
$2$ & $96.79 \pm 0.04$ & $81.75 \pm 0.15$\\
\bottomrule
\end{tabular}
\label{table:different_init_rho}
\end{table}

\subsection{Loss Landscapes}

To illustrate the superior performance of the LETS method, we follow \cite{du2021efficient} to visualize the loss landscapes w.r.t. weight perturbations of SAM, LETS-SAM, ASAM, and LETS-ASAM.
Figure \ref{fig:landscape_resnet18_cifar10} (resp. Figure 8 
in Appendix D.4)
shows the corresponding loss landscapes for different methods built on \textit{ResNet-18} on the \textit{CIFAR-10} (resp. \textit{CIFAR-100}) dataset, respectively.
We can find that the model learned by LETS-SAM (resp. LETS-ASAM) has a flatter loss landscape than that of SAM (resp. ASAM). Since the flatness is a measure for generalization, those results could explain why using the proposed LETS method could lead to performance improvement.

\begin{figure}[!tb]
\centering
\!\!\!
\subfigure[SAM. \label{fig:cifar-10-resnet18-sam}]{\includegraphics[width=0.23\textwidth]{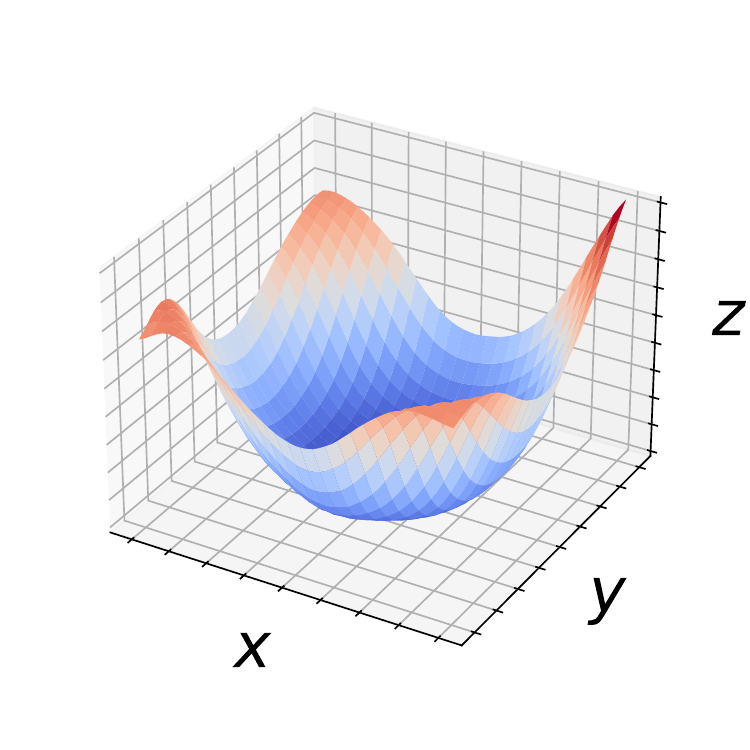}}
\!\!\!
\subfigure[LETS-SAM. \label{fig:cifar-10-resnet18-bsam}]{\includegraphics[width=0.23\textwidth]{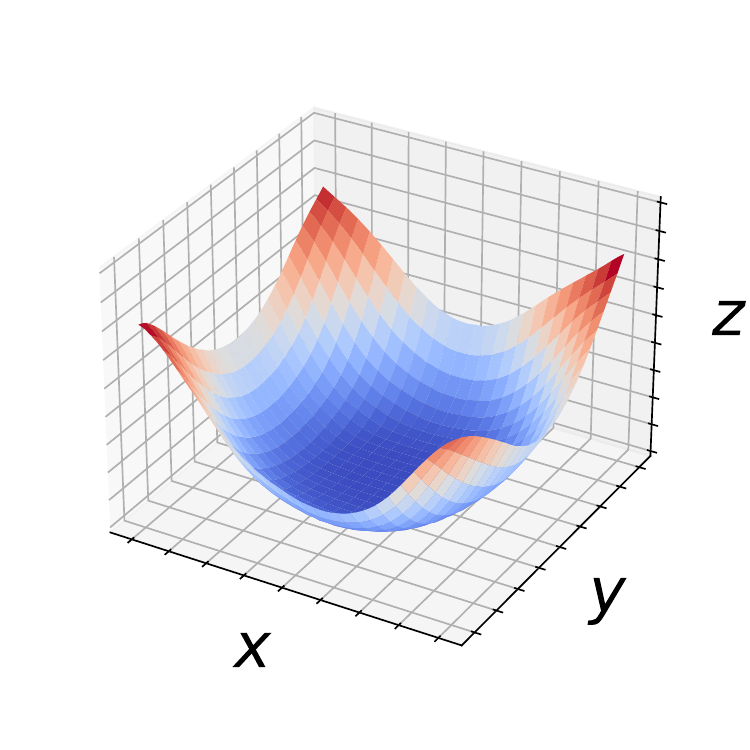}}
\!\!\!
\subfigure[ASAM. \label{fig:cifar-10-resnet18-asam}]{\includegraphics[width=0.23\textwidth]{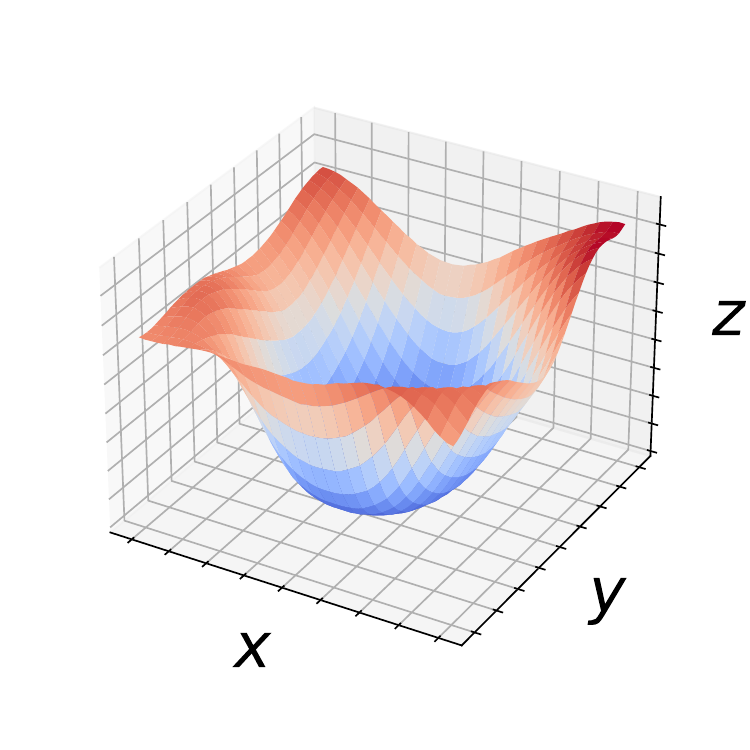}}
\!\!\!
\subfigure[LETS-ASAM. \label{fig:cifar-10-resnet18-basam}]{\includegraphics[width=0.23\textwidth]{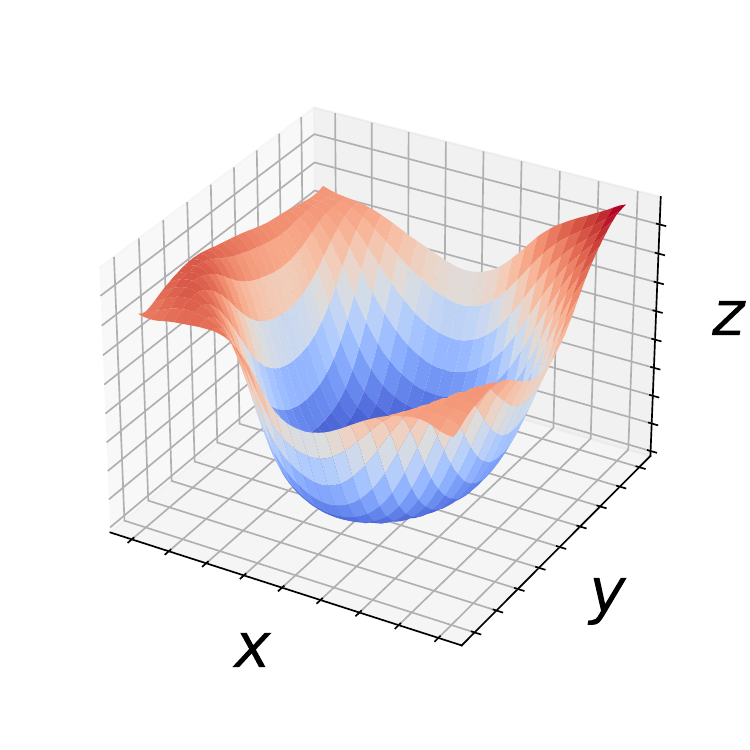}} \!\!\!
\caption{Loss landscapes of different methods built on \textit{ResNet-18} for \textit{CIFAR-10}, where x- and y-axes represent two orthogonal weight perturbations, while z-axis represents the loss value.}
\label{fig:landscape_resnet18_cifar10}
\end{figure}

\subsection{Effects of Generalization Metrics}
\label{sec:abl}

In this section, 
we conduct experiments on the \textit{CIFAR-10} and \textit{CIFAR-100} datasets using \textit{ResNet-18}
to analyze the effects of different generalization metrics (in upper-level problem \eqref{eq:upper-level}),
including validation loss, 
the generalization gap, and its square.
According to results shown in Table \ref{table:differen_matric_sam}, we can find that using  $\frac{1}{2}\left(\hL(\hD^{vl};\vtheta^\star(\rho)) - \hL(\hD^{tr};\vtheta^\star(\rho))\right)^2$ 
achieves the best performance on both datasets, which suggests that it is a good objective for the upper-level problem.

\begin{table}[!t]
    \centering
    
    \caption{Classification accuracy (\%) on \textit{CIFAR-10} and \textit{CIFAR-100} for different generalization metrics on LETS-SAM. 
        The best is in
        \textbf{bold}.} 
    \begin{tabular}{c@{\hskip .07in}c@{\hskip .08in}c}
        \toprule
        &  {\textit{CIFAR-10}} & {\textit{CIFAR-100}}\\
        \midrule
        $\hL(\hD^{vl};\vtheta^\star(\rho))$ & $96.61 \pm 0.07$ & $80.54 \pm 0.06$ \\
        $\hL(\hD^{vl};\vtheta^\star(\rho)) \!-\! \hL(\hD^{tr};\vtheta^\star(\rho))$ & $96.75 \pm 0.18$ & $80.62 \pm 0.15$\\ \rowcolor{Gray}
        $\frac{1}{2}\!\left(\hL(\hD^{vl};\vtheta^\star(\rho)) \!-\! \hL(\hD^{tr};\vtheta^\star(\rho))\right)^2$ & $\mathbf{96.81}\pm0.02$ & $\mathbf{80.71} \pm 0.07$\\
        \bottomrule
    \end{tabular}
    \label{table:differen_matric_sam}
\end{table}

\subsection{Convergence}
In this experiment, we 
study whether the proposed LETS-SAM can converge as suggested in Theorem 4
of Appendix B.
Figure \ref{fig:train-loss-cifar100} (resp. Figure 9
in Appendix D.5)
shows the change of the training loss w.r.t. number of epochs for the experiment on \textit{CIFAR-100} (resp. \textit{CIFAR-10}) in Section \ref{sec:cifar}.
We can find that LETS-SAM and SAM exhibit comparable convergence speeds.
Similarly,
LETS-ASAM and ASAM empirically enjoy similar convergence rates.

\begin{figure}[!t]
    \centering
    \subfigure[\textit{ResNet-18}. \label{fig:train-loss-cifar-100-resnet10}]{\includegraphics[width=0.3\textwidth]{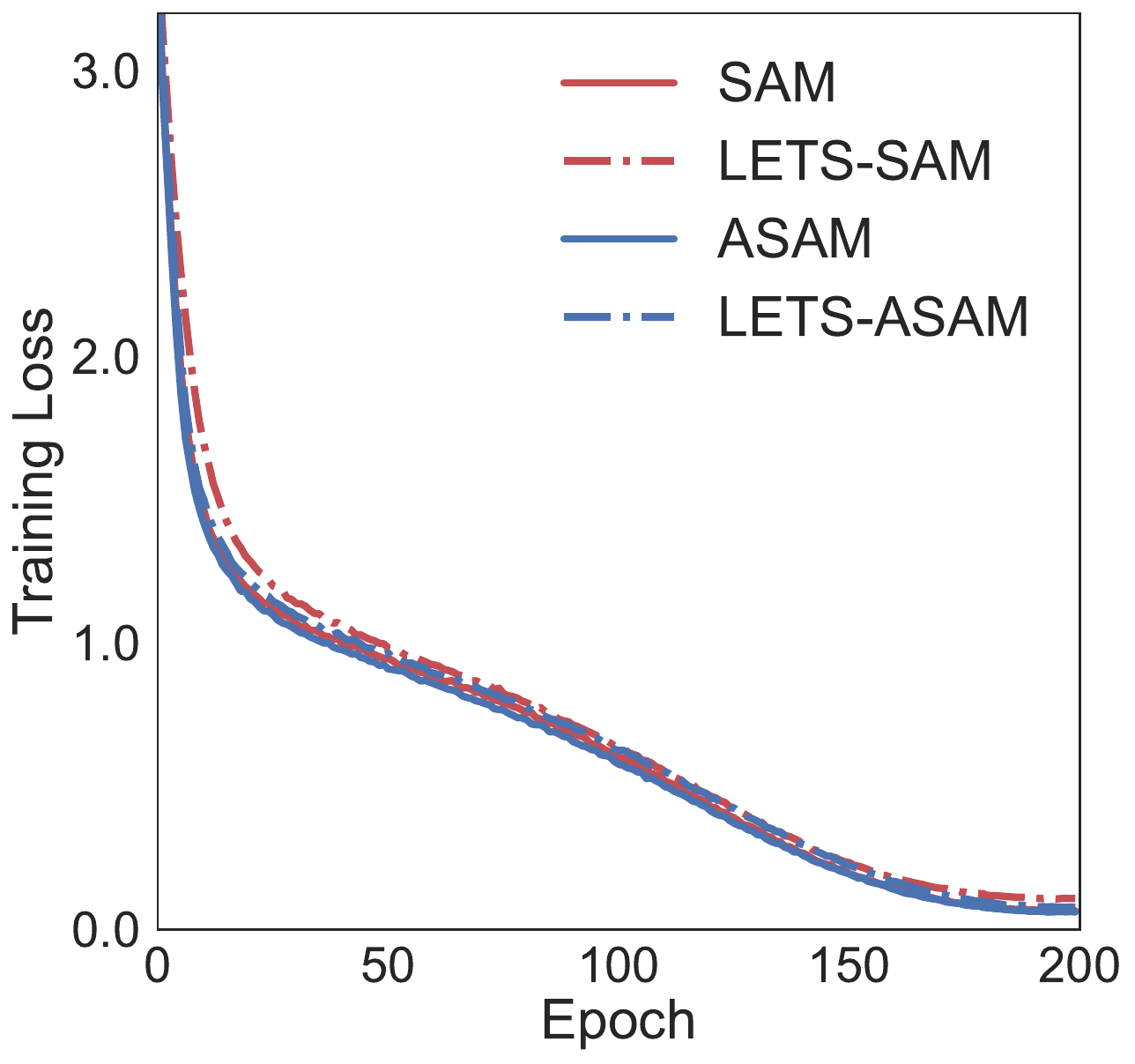}}
    \quad 
    \subfigure[\textit{WideResNet-28-10}. \label{fig:train-loss-cifar-100-wrn}]{\includegraphics[width=0.3\textwidth]{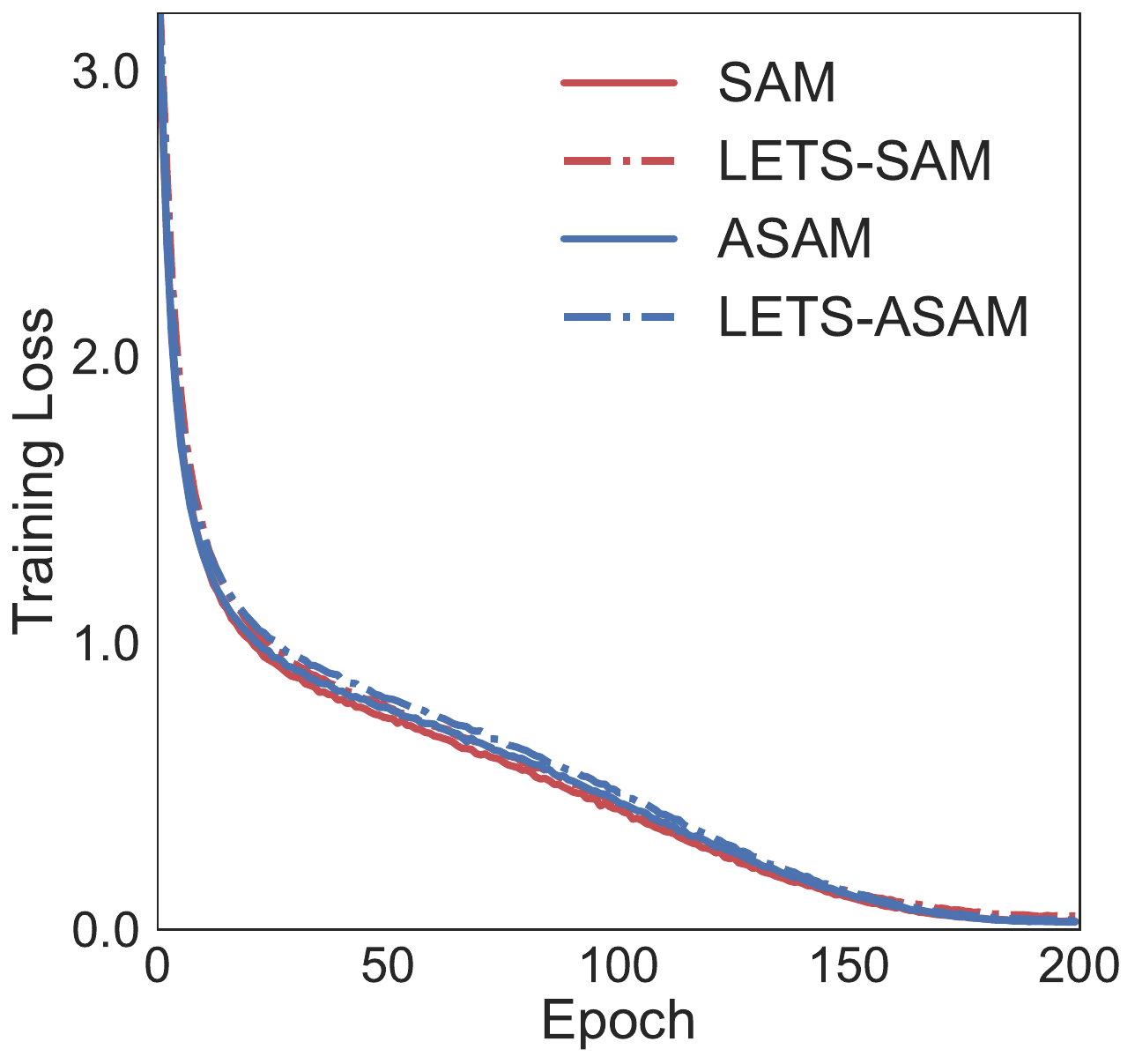}}
    \quad 
    \subfigure[\textit{PyramidNet-110}. \label{fig:train-loss-cifar-100-pym}]{\includegraphics[width=0.3\textwidth]{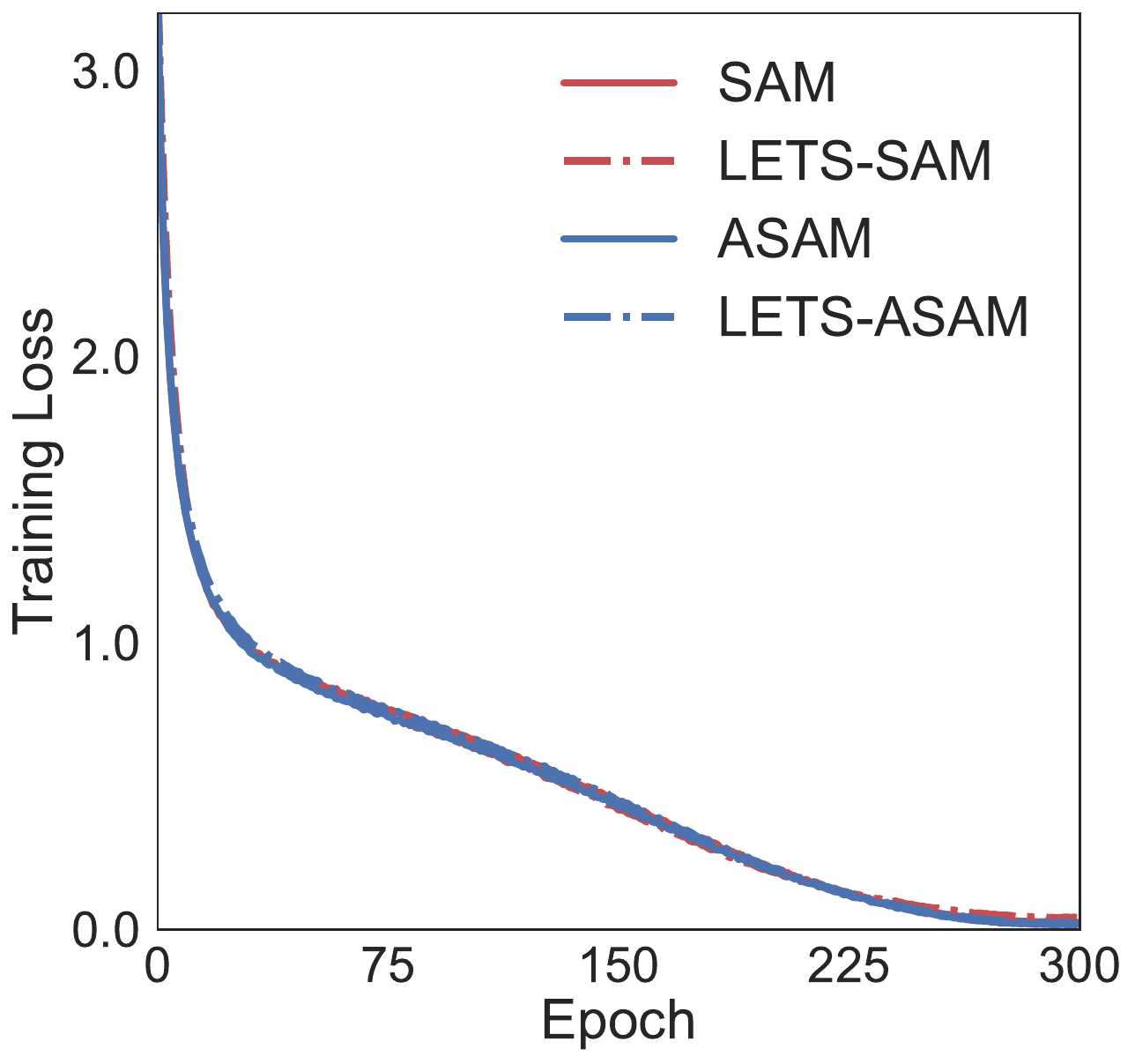}}
    \caption{Training loss
        w.r.t. training epochs
        on \textit{CIFAR-100}. 
        Best viewed in color.
    }
    \label{fig:train-loss-cifar100}
\end{figure}

\section{Conclusion}
In this paper, we study the problem of learning the perturbation radius in sharpness-aware minimization.
The proposed LETS method formulates it as a bilevel optimization problem and proposes a gradient-based algorithm to update the model parameters and the radius alternatively. 
Extensive experiments demonstrate the effectiveness of the proposed LETS method across multiple tasks and various network architectures.
The proposed LETS method is general and can be combined with any SAM algorithm, as shown by the success of LETS-ASAM.

\section*{Acknowledgements}

This work is supported by NSFC key grant under grant no. 62136005, NSFC general grant under grant no. 62076118, and Shenzhen fundamental research program JCYJ20210324105000003.

	%%%%%%%%%%%%%%%%%%%%%%%%%%%%%%%%%%%%%%%%%%%%%%%%%%%%%%%%%%%%
        \onecolumn
	\appendix
 \section{Derivation of gradient descent step for updating $\rho_{t+1}$ in LETS-SAM (step \ref{step:rho-c} in Algorithm \ref{alg:bsam})}
 \label{app:gradient}
		\vskip -.3in
\begin{align*}
\rho_{t+1} \!=& \rho_{t} \!-\! \beta \nabla_{\rho_{t}} \frac{1}{2}(\hL(\hD^{vl};\vtheta_{t+1} (\rho_{t})) - \hL(\hD^{tr};\vtheta_{t+1} (\rho_{t})))^2 \\
\text{(By chain rule)} \!=& \rho_t \!-\! \beta\nabla_{\vtheta_{t+1}}^\top\frac{1}{2}(\hL(\hD^{vl};\vtheta_{t+1}) 
- \hL(\hD^{tr};\vtheta_{t+1}))^2 \nabla_{\rho_{t}}\vtheta_{t+1}\\
\text{(By chain rule and Eq.~(\ref{eq:theta_t}))}\!=& \rho_t \!-\! \beta \nabla_{\vtheta_{t+1}}^\top\frac{1}{2}(\hL(\hD^{vl};\vtheta_{t+1}) 
- \hL(\hD^{tr};\vtheta_{t+1}))^2\\
&\cdot (-\eta \nabla_{\rho_t}\nabla \hL(\left(\hD^{tr};\vtheta_{t}+\rho_{t}\hat{\vepsilon}_{t}^{\text{(SAM)}}\right)))\\
\text{(By chain rule)}\!=&\rho_{t} \!+\! \frac{\beta\eta}{2} \nabla_{\vtheta_{t+1}}^\top(\hL(\hD^{vl};\vtheta_{t+1}) 
- \hL(\hD^{tr};\vtheta_{t+1}))^2 \\
& \cdot \nabla^2 \! \hL\left(\hD^{tr};\vtheta_{t}\!+\!\rho_{t}\hat{\vepsilon}_t^{\text{(SAM)}}\right) \!
\hat{\vepsilon}_{t}^{\text{(SAM)}}.
\end{align*}
 
\section{Theoretical Analysis of Convergence}
	\label{sec:conv}

In this section, we study the convergence of LETS-SAM. 
The following assumptions on the smoothness and bounded variance of stochastic gradients are standard in 
	the literature on nonconvex optimization~\cite{reddi2016stochastic} 
	and SAM~\cite{andriushchenko2022towards,qu2022generalized,jiang2023adaptive}. 
	The assumption on bounded gradients is also standard in SAM~\cite{mi2022make}.
	
	\begin{assumption}[Smoothness]
		\label{ass:smooth}
		$\hL(\hD^{tr}; \vtheta)$ is $\gamma$-smooth in $\vtheta$, 
		i.e., $\|\nabla \hL(\hD^{tr};\vtheta) - \nabla \hL(\hD^{tr};\vtheta') \| \leq \gamma \| \vtheta - \vtheta' \|$.
	\end{assumption}
	\begin{assumption}[Bounded variance of stochastic gradients]
		\label{ass:bd-noise}
		$\bE_{(\vx_i, y_i) \sim \hD^{tr}} \|$ $ \nabla \ell(f(\vx_i;\vtheta), y_i) - \nabla \hL(\hD^{tr}; \vtheta) \|^2 \leq \sigma^2$.
	\end{assumption}
	\begin{assumption}[Bounded gradients]
		\label{ass:grad-norm}
		$\|\nabla \hL(\hD^{tr};\vtheta) \| \leq \zeta$.
	\end{assumption}
	
	Based on those assumptions, we have the following theorem with the proof in Appendix \ref{sec:proof}.
	
	\begin{theorem}\label{thm}
		Let $b$ be the mini-batch size.
		If stepsize $\eta=\frac{1}{\gamma\sqrt{T}}$
		and $\rho_t\leq \frac{\kappa}{\sqrt{T}}$ (where $\kappa>0$ is a constant), Algorithm LETS-SAM satisfies
		\begin{align}
			\min_{0\leq t \leq T-1}  \bE \| \nabla\hL(\hD^{tr}; \vtheta_{t})\|^2
			&\leq \frac{\gamma \bE\hL(\hD^{tr}; \vtheta_{0})}{2 \sqrt{T}} +  \frac{\gamma\zeta \kappa + \frac{\sigma^2}{b} + 2\zeta^2 + \gamma \kappa^2}{2 \sqrt{T}},
		\end{align}
		where the expectation is taken over the
		random training samples.
	\end{theorem}
	
	The $\hO\left(\frac{1}{\sqrt{T}}\right)$ convergence rate in Theorem \ref{thm} is the same as SAM~\cite{andriushchenko2022towards} and its variants~\cite{qu2022generalized,jiang2023adaptive}.
	Hence, adjusting the perturbation radius does not affect the convergence rate.

 \subsection{Proof}
	\label{sec:proof}

\begin{proof}[Proof of Theorem \ref{thm}] \\
To simplify notations, 
		let $\vg_{t+\frac{1}{2}} \equiv\nabla\hL\left(\hB_t^{tr}; \vtheta_{t} +\rho_t \frac{\nabla\hL(\hB_t^{tr}; \vtheta_{t})}{\|\nabla\hL(\hB_t^{tr}; \vtheta_{t})\|}\right)$ and $\vg^{\hD}_{t+\frac{1}{2}} \equiv\nabla\hL\left(\hD^{tr}; \vtheta_{t} +\rho_t \frac{\nabla\hL(\hB_t^{tr}; \vtheta_{t})}{\|\nabla\hL(\hB_t^{tr}; \vtheta_{t})\|}\right)$.

By Taylor expansion and Assumption \ref{ass:smooth}, we have
		\begin{align}
			&\bE\hL(\hD^{tr}; \vtheta_{t+1}) \nonumber\\
			\leq &  \bE\hL(\hD^{tr}; \vtheta_{t}) +
			\bE\nabla^\top\hL(\hD^{tr}; \vtheta_{t}) (\vtheta_{t+1} - \vtheta_t) + \frac{\gamma}{2} \bE\| \vtheta_{t+1} - \vtheta_t\|^2 \nonumber \\
			=&  \bE \hL(\hD^{tr}; \vtheta_{t}) -\eta \bE
			\nabla^\top\hL(\hD^{tr}; \vtheta_{t}) \vg_{t+
				\frac{1}{2}} + \frac{\gamma\eta^2}{2}\bE \|\vg_{t+\frac{1}{2}}\|^2 \nonumber\\
			\leq &  \bE\hL(\hD^{tr}; \vtheta_{t}) -\eta \bE \nabla^\top\hL(\hD^{tr}; \vtheta_{t}) \vg_{t+\frac{1}{2}}
			+ \gamma\eta^2\bE\left( \| \vg_{t+\frac{1}{2}} -\vg^{\hD}_{t+\frac{1}{2}} \|^2 + \|\vg^{\hD}_{t+\frac{1}{2}}\|^2\right)\nonumber \\
			=&  \bE\hL(\hD^{tr}; \vtheta_{t})  -\eta \bE \nabla^\top\hL(\hD^{tr}; \vtheta_{t}) \vg_{t+\frac{1}{2}}+ \frac{\gamma \eta^2 \sigma^2 }{b}+  \gamma\eta^2\bE \|\vg^{\hD}_{t+\frac{1}{2}}\|^2, \label{temp:xaskoi}
		\end{align}
		where the second inequality follows from the property $\|\va\|^2 \leq 2(\| \va -\vb\|^2 + \| \vb\|^2)$ for any two vectors $\va$ and $\vb$.
		We bound the second and last terms separately in the 
		following.
		
		\underline{\textbf{Claim 1}}: $-\bE \nabla^\top\hL(\hD^{tr}; \vtheta_{t}) \vg_{t+\frac{1}{2}} \leq - \bE \| \nabla\hL(\hD^{tr}; \vtheta_{t})\|^2
		+\rho_t \gamma \zeta$.
		
		By triangle inequality, we have
		\begin{align*}
			&\bE \nabla^\top\hL(\hD^{tr}; \vtheta_{t}) \vg_{t+\frac{1}{2}} \\
			&= \bE \nabla^\top\hL(\hD^{tr}; \vtheta_{t}) \left(\nabla\hL(\hD^{tr}; \vtheta_{t}) -  \nabla\hL(\hD^{tr}; \vtheta_{t}) + \vg_{t+\frac{1}{2}}\right) \\
			& = \bE \| \nabla^\top\hL(\hD^{tr}; \vtheta_{t})\|^2
			+ \bE \nabla^\top\hL(\hD^{tr}; \vtheta_{t}) \left(
			\vg_{t+\frac{1}{2}}- \nabla\hL(\hD^{tr}; \vtheta_{t}) \right) \\
			& \geq \bE \| \nabla\hL(\hD^{tr}; \vtheta_{t})\|^2
			- \bE\| \nabla\hL(\hD^{tr}; \vtheta_{t}) \| \| \bE_{\hB_t^{tr}} \left(
			\vg_{t+\frac{1}{2}}- \nabla\hL(\hD^{tr}; \vtheta_{t}) \right) \|\\
			& \geq \bE \| \nabla\hL(\hD^{tr}; \vtheta_{t})\|^2
			-\rho_t \gamma \zeta,	\end{align*}
		where we use Assumption \ref{ass:grad-norm} and $ \| \bE_{\hB_t^{tr}}  (
		\vg_{t+\frac{1}{2}}- \nabla\hL(\hD^{tr}; \vtheta_{t})  ) \| =  \| \bE_{\hB_t^{tr}}  (
		\vg_{t+\frac{1}{2}}- \nabla\hL(\hB_t^{tr}; \vtheta_{t})  ) \| \leq \gamma \rho_t \|\bE_{\hB_t}\frac{\nabla\hL(\hB_t^{tr}; \vtheta_{t})}{\|\nabla\hL(\hB_t^{tr}; \vtheta_{t})\|} \| \leq  \gamma\rho_t$ to obtain the last inequality.
		
		\underline{\textbf{Claim 2}}: $\bE\|\vg^{\hD}_{t+\frac{1}{2}}\|^2  \leq 2\zeta^2
		+
		\gamma\rho_t^2 - \bE  \| \nabla\hL(\hD^{tr}; \vtheta_{t})  \|^2$.
		
		Using the property $\| \va\|^2 = \| \va - \vb\|^2 - \| \vb\|^2 + 2 \va^\top \vb$, it follows that
		\begin{align}
			\bE\|\vg^{\hD}_{t+\frac{1}{2}}\|^2 
			=& 
			\bE\left( \| \vg^{\hD}_{t+\frac{1}{2}} - \nabla\hL(\hD^{tr}; \vtheta_{t}) \|^2 -  \| \nabla\hL(\hD^{tr}; \vtheta_{t})  \|^2 +
			2\nabla^\top\hL(\hD^{tr}; \vtheta_{t}) \vg^{\hD}_{t+\frac{1}{2}} \right) \nonumber\\
			=&2\bE\nabla^\top\hL(\hD^{tr}; \vtheta_{t})
			\vg^{\hD}_{t+\frac{1}{2}}
			+
			\bE \| \vg^{\hD}_{t+\frac{1}{2}}- \nabla\hL(\hD^{tr}; \vtheta_{t}) \|^2 - \bE  \| \nabla\hL(\hD^{tr}; \vtheta_{t})  \|^2. \label{temp:asdxk}
		\end{align}
		For the first term, 
		it follows that
		\begin{align}
			\bE\nabla^\top\hL(\hD^{tr}; \vtheta_{t})
			\vg^{\hD}_{t+\frac{1}{2}} \leq  
			\bE \|\nabla\hL(\hD^{tr}; \vtheta_{t})\|
			\|	\vg^{\hD}_{t+\frac{1}{2}}\| \leq \zeta^2,
			\label{eq: tempsxa}
		\end{align}
		where we have used Assumption \ref{ass:grad-norm} to obtain the last inequality.
		
		For the second term in Eq.~\eqref{temp:asdxk},
		by Assumption \ref{ass:smooth},
		it follows that
		\begin{align}
			\bE \|  \vg^{\hD}_{t+\frac{1}{2}} - \nabla\hL(\hD^{tr}; \vtheta_{t}) \|^2 \leq \gamma \| \vtheta_{t} +\rho_t \frac{\nabla\hL(\hB_t^{tr}; \vtheta_{t})}{\|\nabla\hL(\hB_t^{tr}; \vtheta_{t})\|} - \vtheta_{t} \|^2  = \gamma \rho^2_t. \label{temp:exao}
		\end{align}	
		Substituting \eqref{eq: tempsxa}  and \eqref{temp:exao}
		into \eqref{temp:asdxk}, we have
		\begin{align*}
			\bE\|\vg^{\hD}_{t+\frac{1}{2}}\|^2  \leq 2\zeta^2
			+
			\gamma\rho_t^2 - \bE  \| \nabla\hL(\hD^{tr}; \vtheta_{t})  \|^2.
		\end{align*}
		
		% \newpage
		Using Claims 1 and 2, for \eqref{temp:xaskoi},
		we have 
		\begin{align*}
			&\bE\hL(\hD^{tr}; \vtheta_{t+1}) \\
			&\leq \bE\hL(\hD^{tr}; \vtheta_{t})
			- \eta (1+\gamma\eta ) \bE \| \nabla\hL(\hD^{tr}; \vtheta_{t})\|^2
			+\eta\rho_t \gamma \zeta
			+ \frac{\gamma \eta^2 \sigma^2}{b}
			+ \gamma \eta^2 (2\zeta^2
			+
			\gamma\rho_t^2) \\
			&\leq  \bE\hL(\hD^{tr}; \vtheta_{t}) 
			-  2\eta \bE \| \nabla\hL(\hD^{tr}; \vtheta_{t})\|^2
			+ \frac{\gamma\zeta \kappa+ \frac{\sigma^2}{b} + 2\zeta^2 + \gamma \kappa^2}{\gamma T}.
		\end{align*}
		Summing over both sides from $t=1$ to $T$, and rearranging it, we have
		\begin{align*}
			2\eta \sum_{t=1}^{T} \bE \| \nabla\hL(\hD^{tr}; \vtheta_{t})\|^2 \leq \bE\hL(\hD^{tr}; \vtheta_{0}) + \frac{\gamma\zeta \kappa  + \frac{\sigma^2}{b} + 2\zeta^2 + \gamma \kappa^2}{\gamma}
		\end{align*}
		Hence, 
		\begin{align*}
			\min_{0\leq t \leq T-1}  \bE \| \nabla\hL(\hD^{tr}; \vtheta_{t})\|^2
			& \leq \frac{1}{T}  \sum_{t=1}^{T} \bE \| \nabla\hL(\hD^{tr}; \vtheta_{t})\|^2 \\
			& \leq \frac{\gamma \bE\hL(\hD^{tr}; \vtheta_{0})}{2 \sqrt{T}} +  \frac{\gamma\zeta \kappa + \frac{\sigma^2}{b} + 2\zeta^2 + \gamma \kappa^2}{2 \sqrt{T}},
		\end{align*}
		and we finish the proof.
	\end{proof}

% \newpage
\section{\textit{GLUE} Benchmark}
\label{app:dataset}

The General Language Understanding Evaluation (\textit{GLUE}) benchmark is a collection of resources for training, evaluating, and analyzing natural language understanding systems. \textit{GLUE} consists of nine language understanding corpora including natural language inference, question answering, paraphrase detection, sentiment analysis, linguistic acceptability, and text similarity. \textit{GLUE} covers a diverse range of dataset sizes, text genres, and degrees of difficulty. The task and metric of each corpus are shown in Table \ref{tab:GLUE-dataset}.

\begin{table}[!ht]
    \centering
    \caption{Details of the \textit{GLUE} benchmark.}
    \vskip -.1in
    \begin{tabular}{c|c|c}
        \toprule
        Corpus & Task & Metric \\
        \midrule
         CoLA & Acceptability &  Matthews Correlation\\
         SST & Sentiment & Accuracy\\
         MRPC & Paraphrase & Accuracy/F1\\
         STSB & Similarity & Pearson/Spearmanr\\
         QQP & Paraphrase & Accuracy/F1\\
         MNLI & NLI & Accuracy\\
         QNLI & QA/NLI & Accuracy\\
         RTE & NLI & Accuracy\\
         \bottomrule
    \end{tabular}
		\vskip -.1in
    \label{tab:GLUE-dataset}
\end{table}

 \section{Additional Experimental Results}
	\label{app:expt}
\subsection{\textit{IWSLT'14 DE-EN}}
The BLEU scores on \textit{IWSLT'14 DE-EN} are shown in Table \ref{tab:iwslt}.
We can find that
LETS-SAM is better than SAM, while 
LETS-ASAM outperforms ASAM,
demonstrating the effectiveness of the proposed LETS method.

\begin{table}[!ht]
		\centering
		\caption{BLEU scores on \textit{IWSLT'14 DE-EN}. The best result in each comparison
    group is \underline{underlined} and the best result across all the groups is in
    \textbf{bold}.}
   
		\vskip -.1in
		\begin{tabular}{cc}
			\toprule
            ERM  & $34.60 \pm 0.04$\\
            \midrule
            SAM & $34.70 \pm 0.07$\\ \rowcolor{Gray}
            LETS-SAM & $\mathbf{\underline{34.85}} \pm 0.02$\\
            \midrule
            ASAM & $34.60 \pm 0.13$\\ \rowcolor{Gray}
            LETS-ASAM & $\underline{34.72} \pm 0.10$\\
			\bottomrule
		\end{tabular}
		\vskip -.1in
		\label{tab:iwslt}
	\end{table}
 
\subsection{\textit{GLUE}} 
The results on eight corpora of \textit{GLUE} are shown in Table \ref{tab:glue}.
We can find that LETS-SAM is consistently better than SAM. Compared with ASAM, LETS-ASAM always has better performance, demonstrating the effectiveness of the proposed LETS. Furthermore, on \textit{SST}, \textit{MRPC}, \textit{STSB}, \textit{MNLI}, and \textit{QNLI} datasets,
LETS-SAM achieves the highest accuracy,
while LETS-SAM is the best on the \textit{CoLA}, \textit{QQP} and \textit{RTE} datasets.

  \begin{table*}[!t]
		\centering
		\caption{Experimental results on the GLUE development set. The better result in each comparison
    group is \underline{underlined} and the best result across all the groups is in
    \textbf{bold}.}
		\vskip -.1in
		\resizebox{0.98\textwidth}{!}{
		\begin{tabular}{ccccccccc}
			\toprule
              & CoLA & SST & MRPC & STSB & QQP & MNLI & QNLI & RTE\\
              \midrule
            ERM  & $64.37 \pm 0.27$ & $94.99 \pm 0.07$ & $89.95 \pm 0.42$ & $91.17 \pm 0.42$ & $91.77 \pm 0.08$ & $88.48 \pm 0.10$ & $93.41 \pm 0.37$ & $63.54 \pm 9.93$\\
            \midrule
            SAM & $64.40 \pm 0.40$ & $95.22 \pm 0.17$ & $90.36 \pm 0.57$ & $91.22 \pm 0.19$ & $92.04 \pm 0.17$ & $88.67 \pm 0.24$ & $93.58 \pm 0.05$ & $56.44 \pm 4.71$\\ \rowcolor{Gray}
            LETS-SAM & $\underline{65.17} \pm 0.81$ & $\mathbf{\underline{95.57}} \pm 0.06$ & $\mathbf{\underline{90.85}} \pm 0.14$ & $\mathbf{\underline{91.51}} \pm 0.06$ & $\underline{92.30} \pm 0.03$ & $\mathbf{\underline{88.88}} \pm 0.05$ & $\mathbf{\underline{93.90}} \pm 0.12$ & $\underline{74.01} \pm 2.96$\\
            \midrule
            ASAM & $63.89 \pm 0.63$ & $95.26 \pm 0.24$ & $90.28 \pm 0.28$ & $91.23 \pm 0.23$ & $92.14 \pm 0.11$ & $88.51 \pm 0.11$ & $93.51 \pm 0.04$ & $69.80 \pm 6.18$\\ \rowcolor{Gray}
            LETS-ASAM & $\mathbf{\underline{65.81}} \pm 0.44$ & $\underline{95.37} \pm 0.06$ & $\underline{90.52} \pm 0.37$ & $\underline{91.48} \pm 0.15$ & $\mathbf{\underline{92.34}} \pm 0.09$ & $\underline{88.82} \pm 0.20$ & $\underline{93.80} \pm 0.04$ & $\mathbf{\underline{74.85}} \pm 4.15$\\
			\bottomrule
		\end{tabular}
  }
		\vskip -.2in
		\label{tab:glue}
	\end{table*}

\subsection{Generalization Gap}
\label{app:gene_gap}
\begin{figure}[!t]
    \centering
    \!\!\!
    \subfigure[\textit{ResNet-18}. \label{fig:cifar-10-resnet10}]{\includegraphics[width=0.25\textwidth]{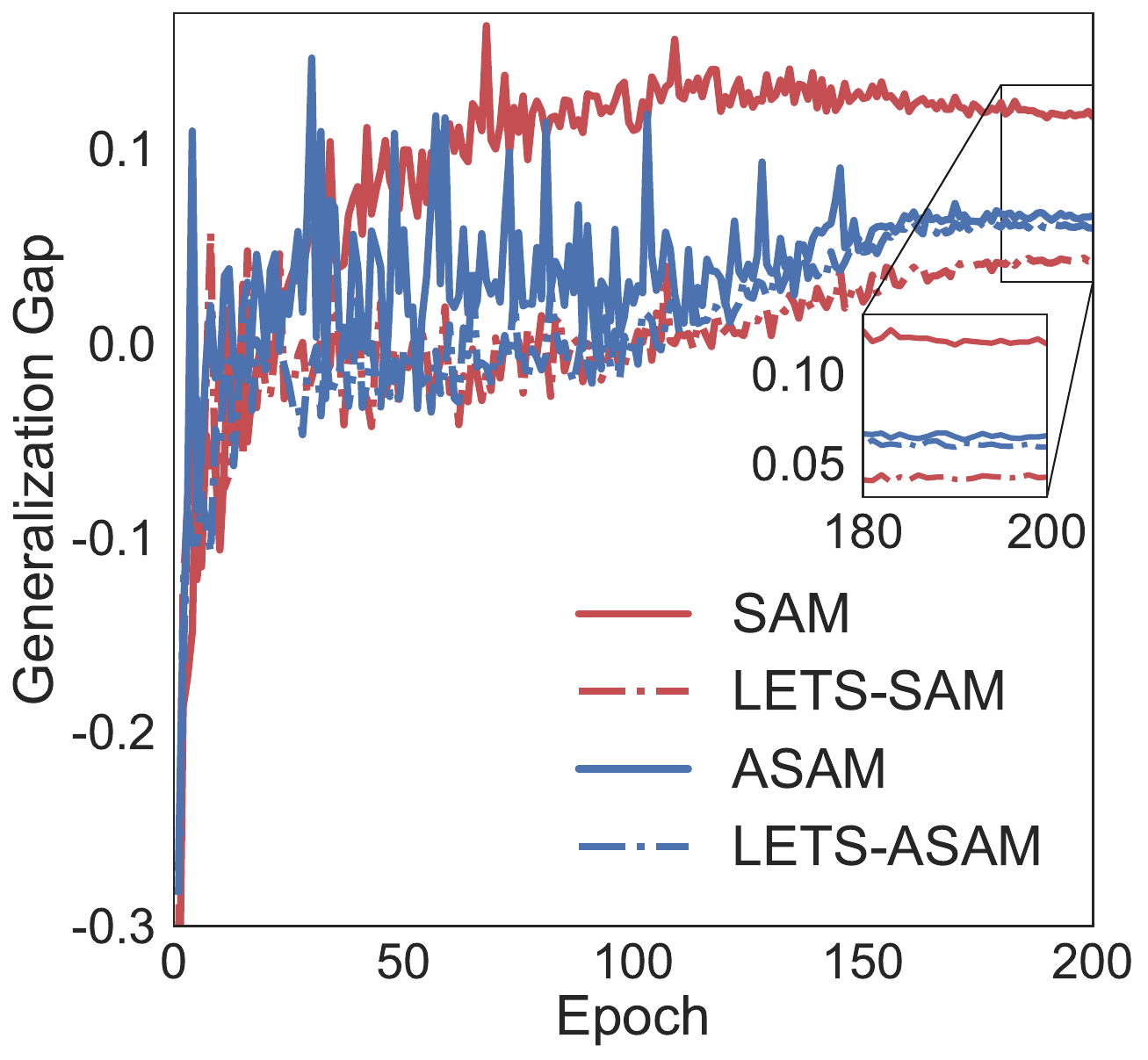}}
    \!\!\!
    \subfigure[\!\textit{WideResNet-28-10}. \label{fig:cifar-10-wrn}]{\includegraphics[width=0.25\textwidth]{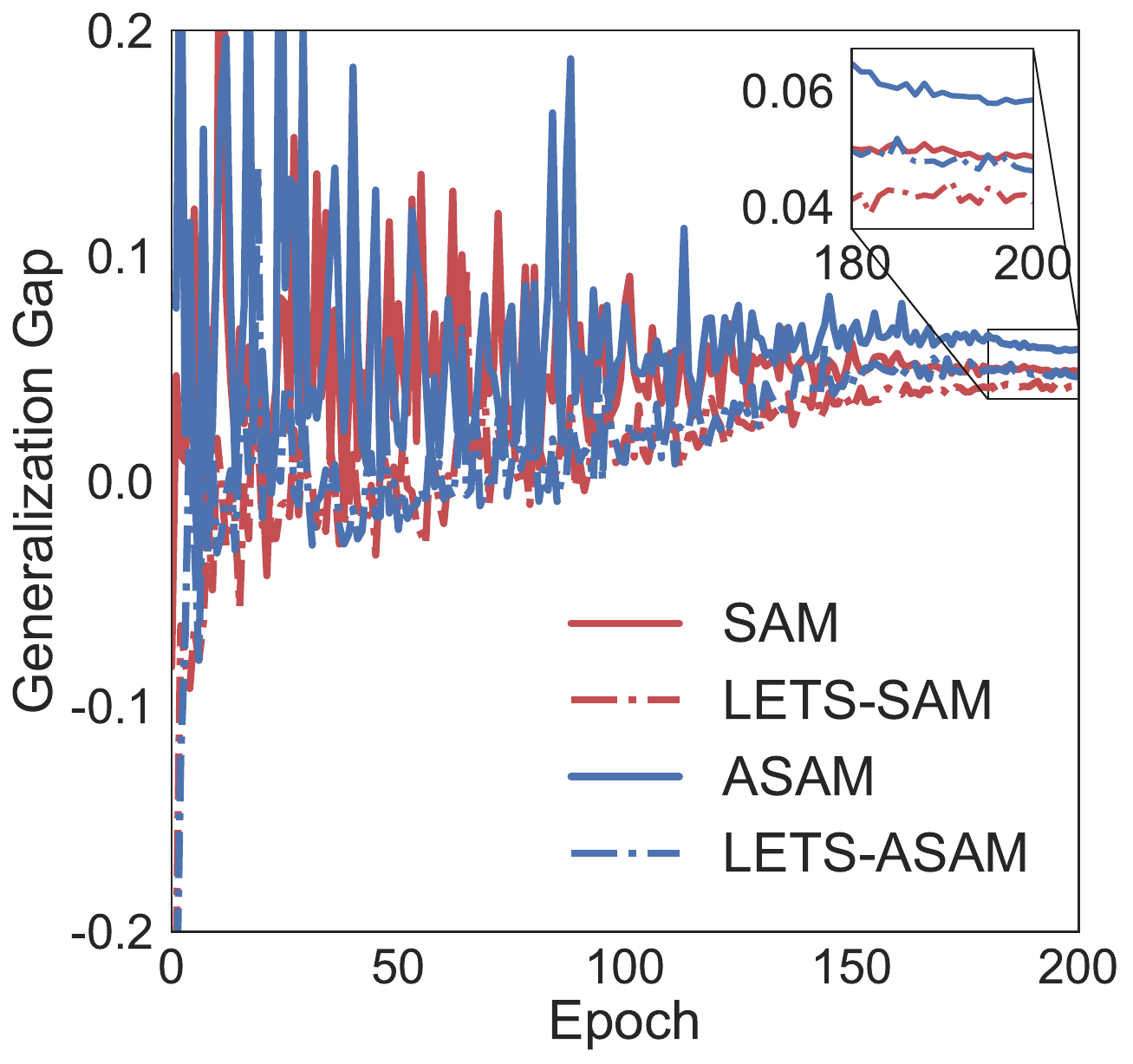}}
    \!
    \subfigure[\textit{PyramidNet-110}. \label{fig:cifar-10-pym}]{\includegraphics[width=0.25\textwidth]{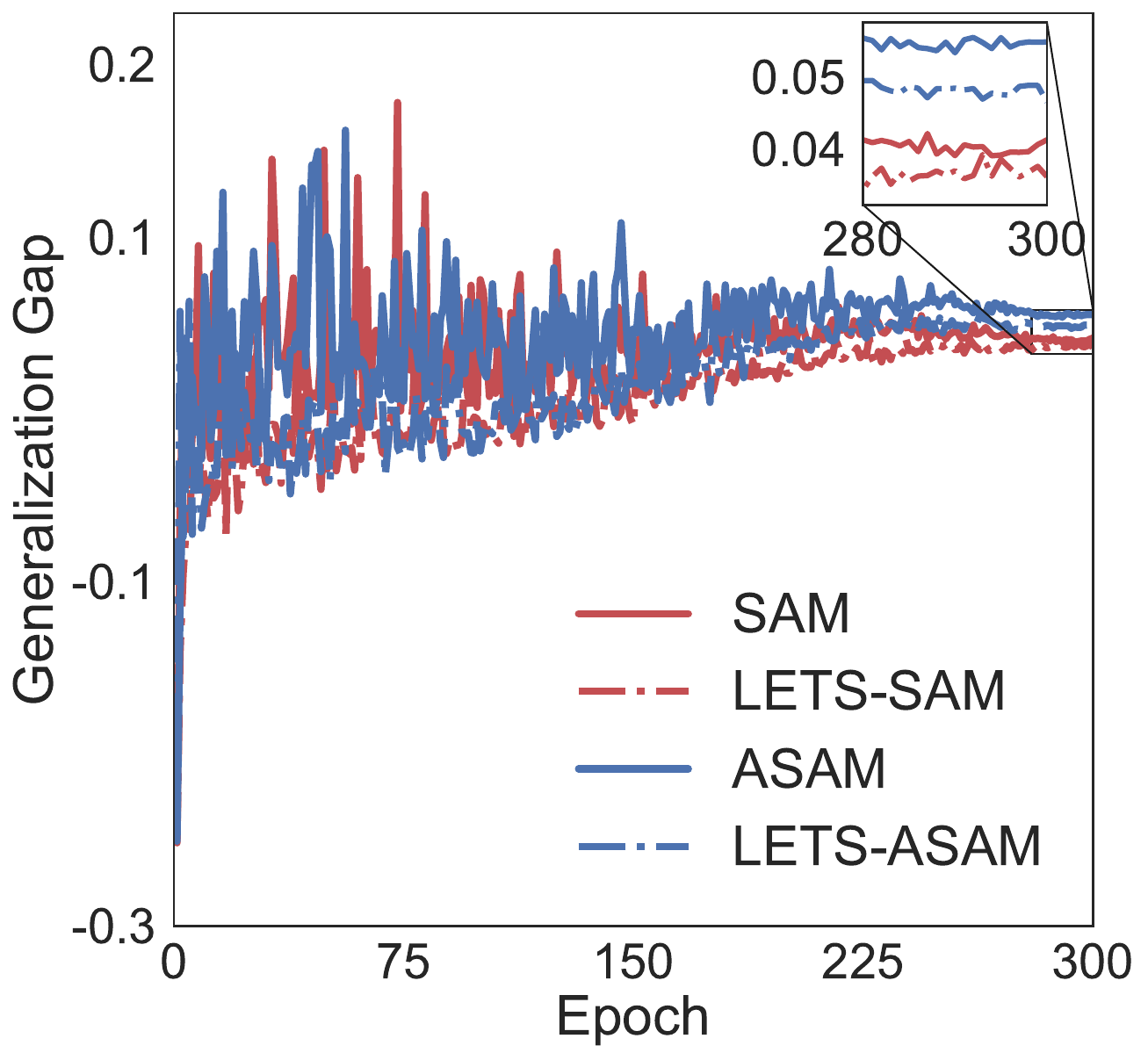}}
    \!\!\!
    \vskip -.25in
    \caption{Generalization gap
        w.r.t. training epochs
        on \textit{CIFAR-10}. 
        Best viewed in color.
    }
    \label{fig:generalization-gap-cifar10}
\end{figure}
Figure \ref{fig:generalization-gap-cifar10} show the generalization gap (i.e., $\hL(\hD^{ts}; \vtheta_t) - \hL(\hD^{tr}; \vtheta_t)$) w.r.t. training epochs on \textit{CIFAR-10} dataset. 
As shown,
LETS-SAM (resp. LETS-ASAM) has a smaller generalization gap than SAM (resp. ASAM) when the training process nearly converges,
verifying that learning the perturbation radius can reduce the generalization gap.

\subsection{Loss Landscapes}
\label{app:loss_lanscapes}

To illustrate the superior performance of the LETS method, we follow \cite{du2021efficient} to visualize the loss landscapes w.r.t. weight perturbations of SAM, LETS-SAM, ASAM, and LETS-ASAM.
Figure \ref{fig:landscape_resnet18_cifar100} shows the corresponding loss landscapes for different methods built on \textit{ResNet-18} on the \textit{CIFAR-100} dataset.
We can find that the model learned by LETS-SAM (resp. LETS-ASAM) has a flatter loss landscape than that of SAM (resp. ASAM). Since the flatness is a measure for generalization, those results could explain why using the proposed LETS method could lead to performance improvement.

\begin{figure}[!t]
\centering
\vskip -.1in
\!\!\!
\subfigure[SAM. \label{fig:cifar-100-resnet18-sam}]{\includegraphics[width=0.2\textwidth]{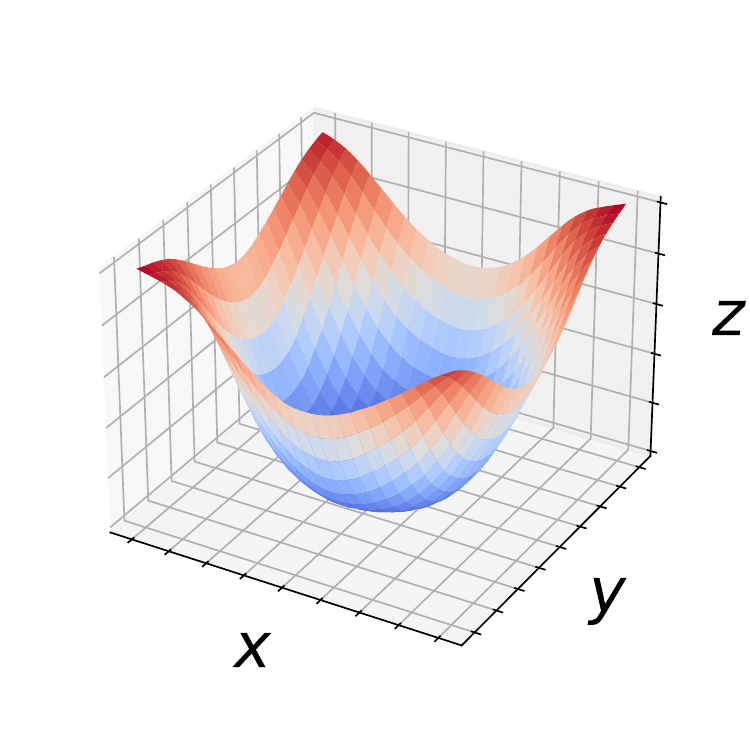}}
\!\!\!
\subfigure[LETS-SAM. \label{fig:cifar-100-resnet18-bsam}]{\includegraphics[width=0.2\textwidth]{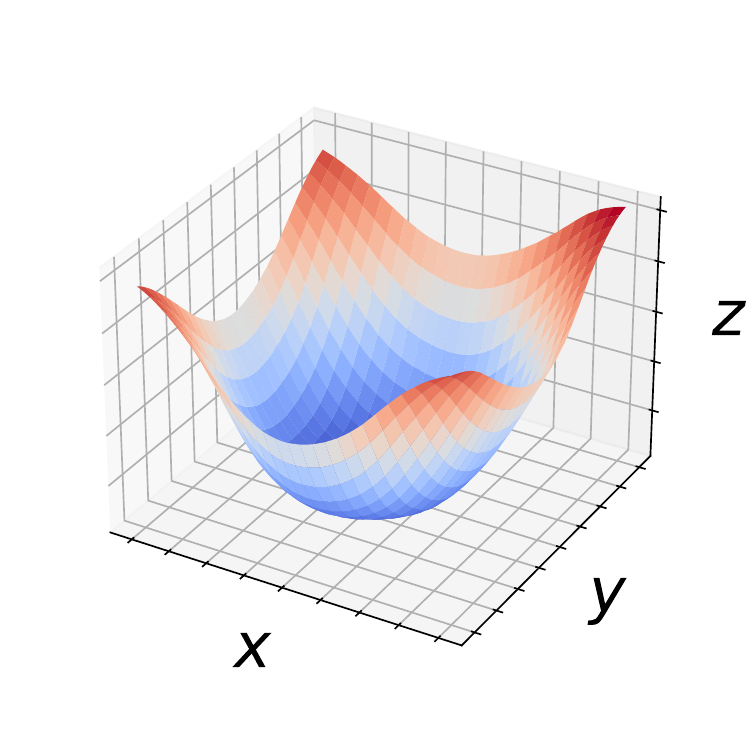}}
\!\!\!
\subfigure[ASAM. \label{fig:cifar-100-resnet18-asam}]{\includegraphics[width=0.2\textwidth]{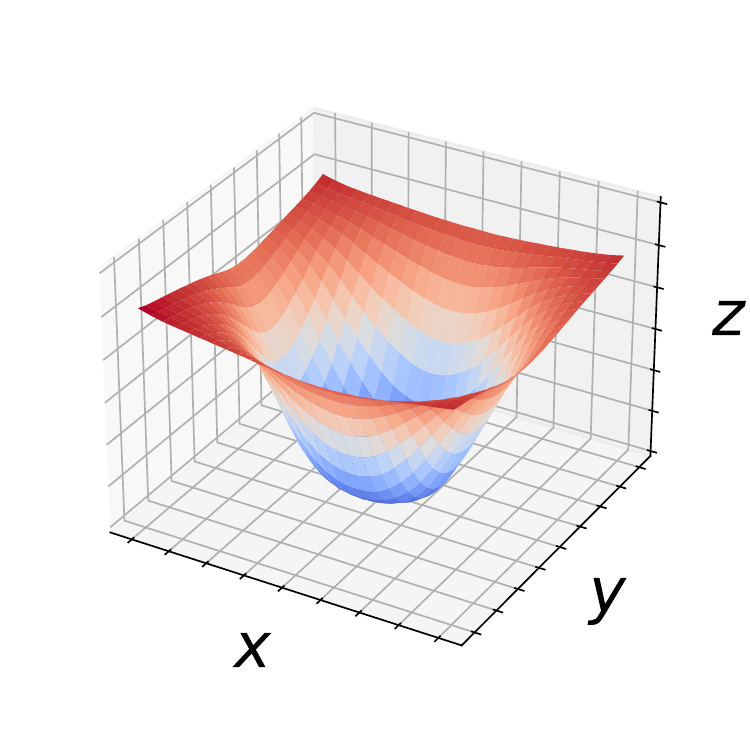}}
\!\!\!
\subfigure[LETS-ASAM. \label{fig:cifar-100-resnet18-basam}]{\includegraphics[width=0.2\textwidth]{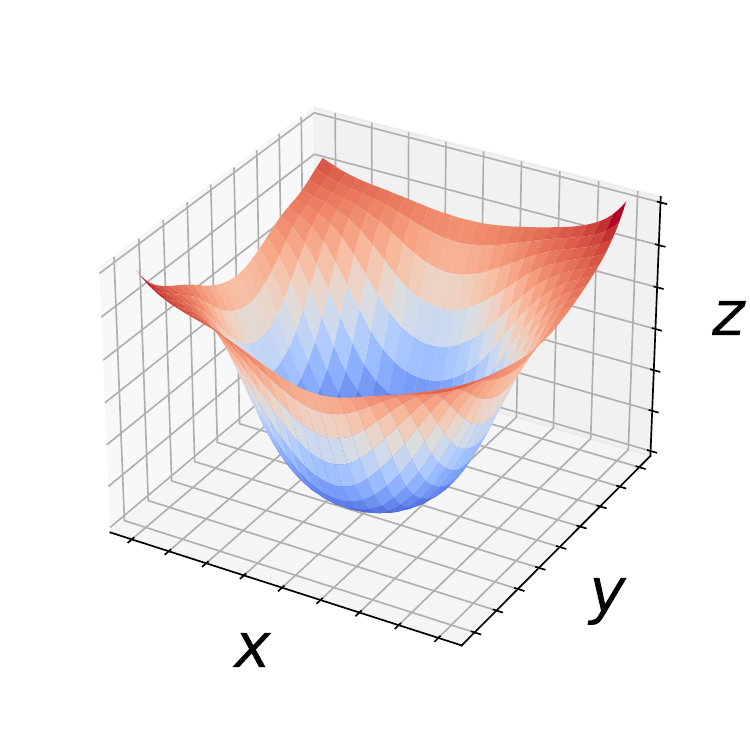}}
\!\!\!
\vskip -.25in
\caption{Loss landscapes of different methods built on \textit{ResNet-18} for \textit{CIFAR-100}, where x- and y-axes represent two orthogonal weight perturbations, while z-axis represents the loss value.}
\label{fig:landscape_resnet18_cifar100} 
\vskip -.15in
\end{figure}

\subsection{Convergence}
\label{app:convergence}

In this experiment, we 
study whether the proposed LETS-SAM can converge as suggested in Theorem \ref{thm}.
Figure \ref{fig:train-loss-cifar10} shows the change of the training loss w.r.t. the number of epochs for the experiment on \textit{CIFAR-10} in Section \ref{sec:cifar}.
As can be seen, LETS-SAM and SAM empirically enjoy similar convergence rates. 
Moreover, LETS-ASAM exhibits a similar convergence behavior to ASAM.

\begin{figure}[!ht]
    \centering
    \!\!\!
    \subfigure[\textit{ResNet-18}. \label{fig:train-loss-cifar-10-resnet10}]{\includegraphics[width=0.3\textwidth]{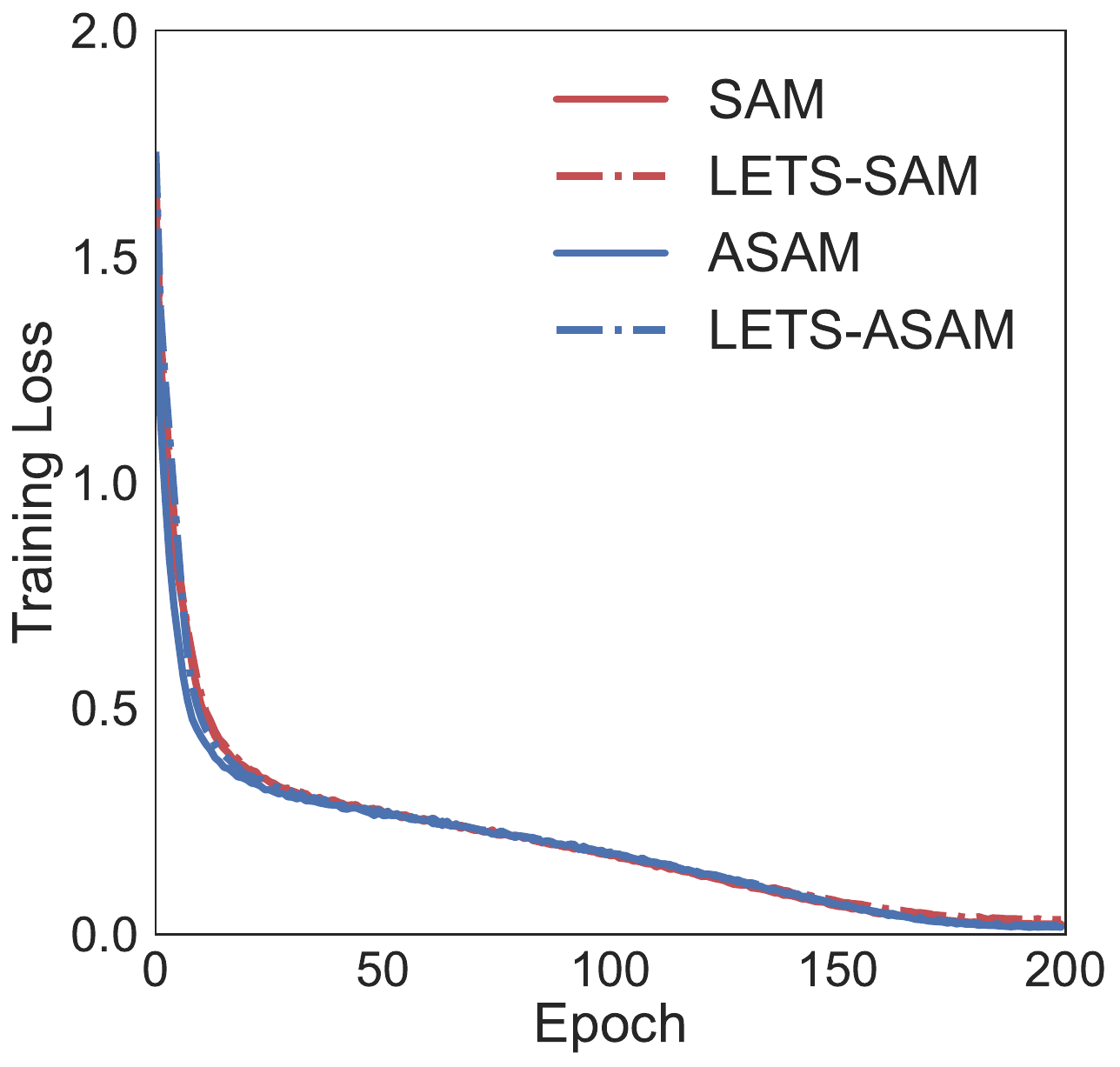}}
    \!\!\!
    \subfigure[\!\textit{WideResNet-28-10}. \label{fig:train-loss-cifar-10-wrn}]{\includegraphics[width=0.3\textwidth]{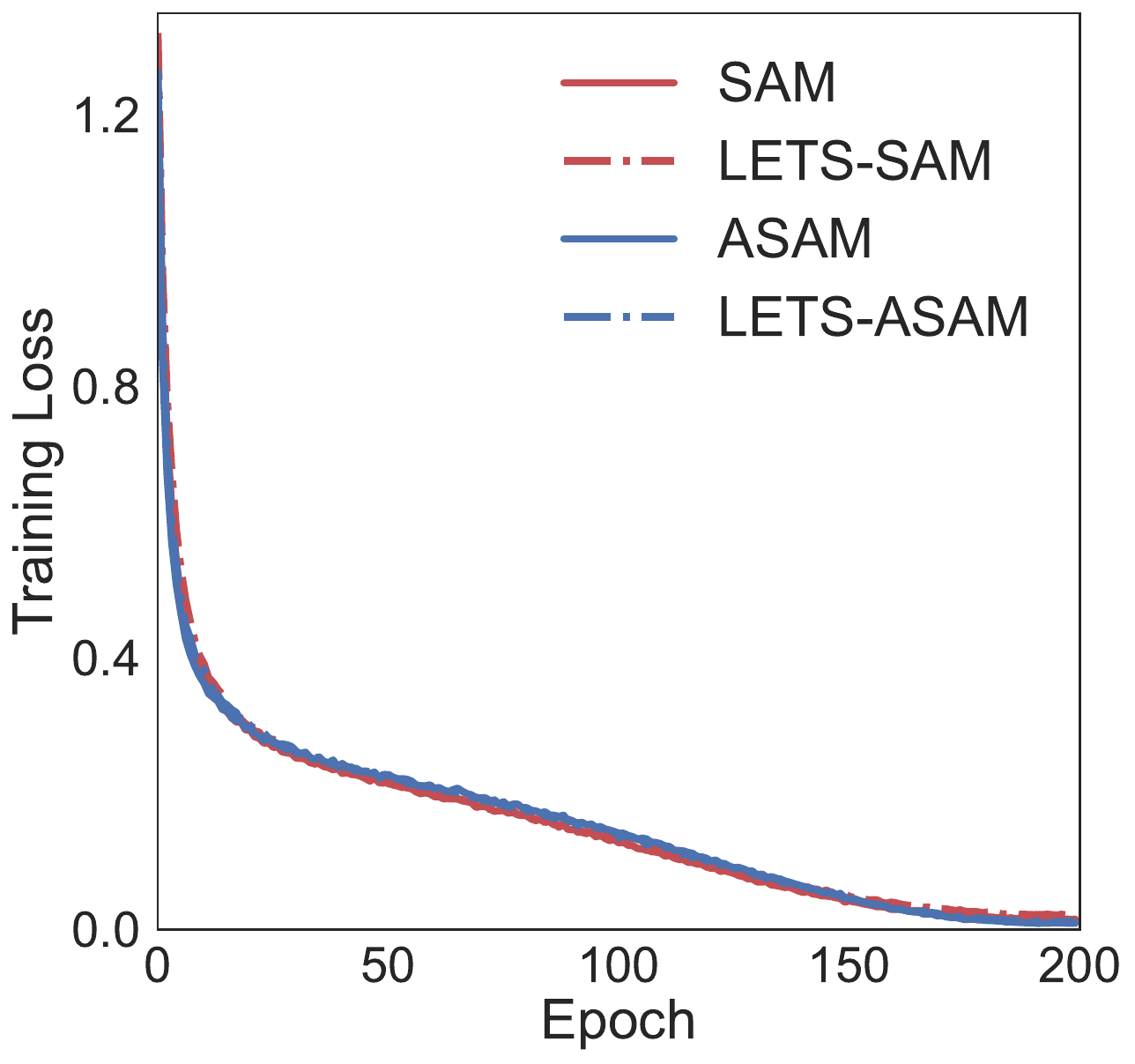}}
    \!
    \subfigure[\textit{PyramidNet-110} \label{fig:train-loss-cifar-10-pym}]{\includegraphics[width=0.3\textwidth]{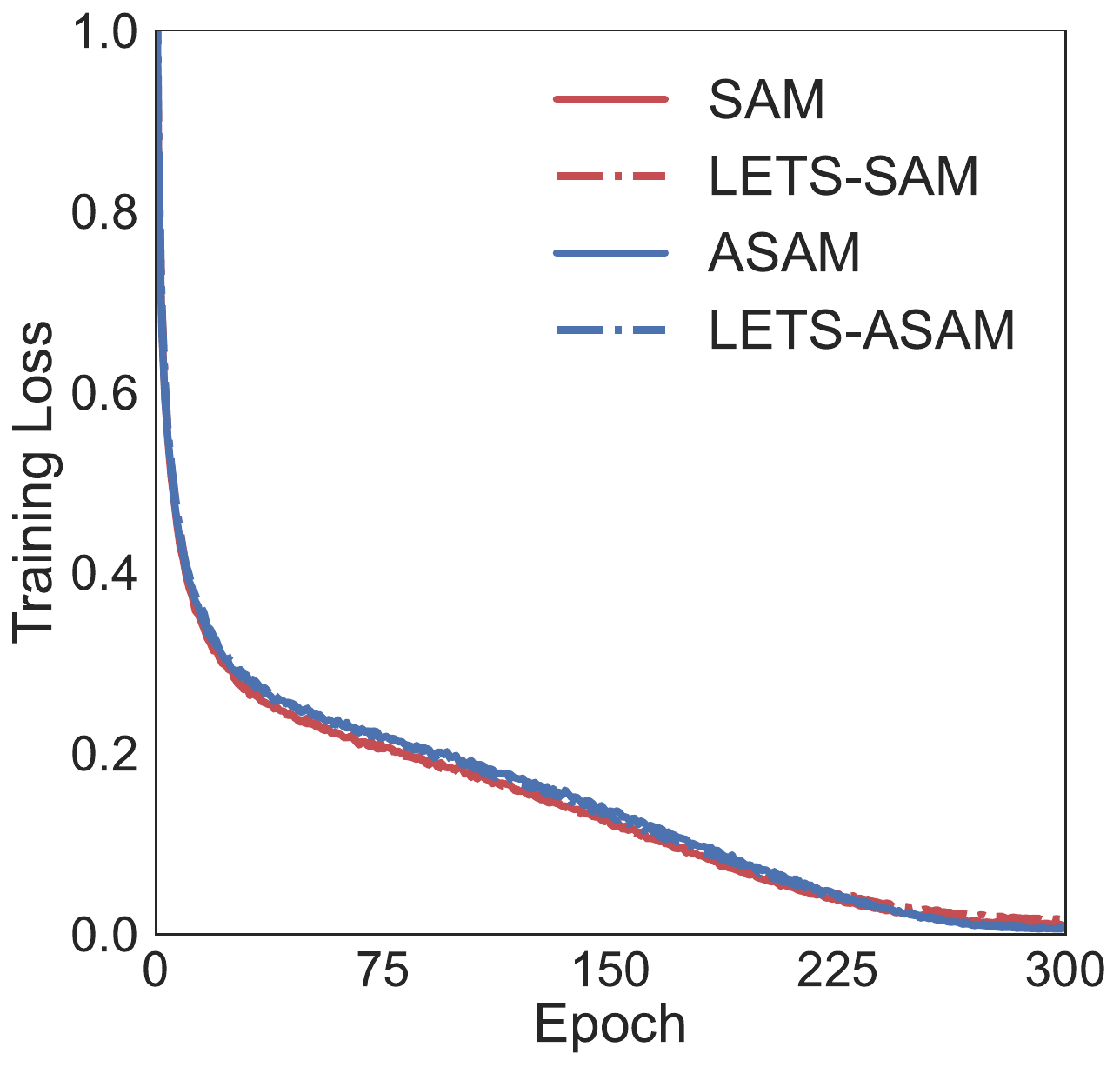}}
    \!\!\!
    \vskip -.15in
    \caption{Training loss w.r.t. training epochs
        on \textit{CIFAR-10}. 
        Best viewed in color.
    }
    \label{fig:train-loss-cifar10}
		\vskip -.3in
\end{figure}
	
	\section{Implementation Details}
	\label{app:hyp}
	
	Experiments on \textit{CIFAR-10} and \textit{CIFAR-100} are conducted on NVIDIA GeForce RTX $3090$ GPUs. 
	Experiments on \textit{ImageNet}, \textit{IWSLT'14 DE-EN} and \textit{GLUE} are conducted on NVIDIA A$100$ GPUs.

 Table \ref{table:impl-cifar} presents the hyperparameters employed in CV tasks. 
 For CV tasks, following experimental setups in \cite{foret2021sharpness,kwon2021asam,jiang2023adaptive}, both SAM (resp. ASAM) and LETS-SAM (resp. LETS-ASAM) use the same hyperparameter configurations. 
  Tables \ref{table:impl-glue} and \ref{table:impl-iwslt} present the hyperparameters employed on the \textit{GLUE} and \textit{IWSLT’14 DE-EN} datasets, respectively. 
  For the \textit{GLUE} dataset, we adopt the experimental setups in \cite{he2020deberta}. For the \textit{IWSLT'14 DE-EN} dataset, we follow the experimental setups used in \cite{kwon2021asam}. % ott2019fairseq

\begin{table}[!ht]
		\centering
		\vskip -.25in
		\caption{Hyperparameters for CV datasets.}
		\vskip -.15in
  \resizebox{0.7\textwidth}{!}{
		\begin{tabular}{c c c}
			\toprule
			& LETS-SAM & LETS-ASAM \\
			\midrule
			\textit{model parameters} & \\
			optimizer & \multicolumn{2}{c}{SGD optimizer for \textit{ResNet-18}, \textit{WideResNet-28-10}, \textit{PyramidNet-110} and \textit{ResNet-50}, Adam optimizer for \textit{ViT-S16}} \\
			initial learning rate & \multicolumn{2}{c}{0.1 for \textit{ResNet-18}, \textit{WideResNet-28-10}, \textit{PyramidNet-110} and \textit{ResNet-50}, 0.0001 for \textit{ViT-S16}}\\
			initialization & \textit{He Initialization}~\cite{he2015delving} & \textit{He Initialization}~\cite{he2015delving}\\ 
			weight decay & \multicolumn{2}{c}{0.0005 for \textit{ResNet-18}, \textit{WideResNet-28-10} and \textit{PyramidNet-110}, 0.0001 for \textit{ResNet-50}}\\
			momentum & 0.9 & 0.9\\
			learning rate schedule & Cosine & Cosine \\
			\midrule 
			\textit{perturbation radius $\rho$} & \\
			optimizer   & Adam optimizer & Adam optimizer\\
			Adam $\beta_1$ & $0.9$ & $0.9$ \\
			Adam $\beta_2$ & $0.999$ & $0.999$ \\
			initial learning rate & 0.0001 & 0.0001 \\
			learning rate schedule & Exponential & Exponential\\
			\midrule
			batch size & \multicolumn{2}{c}{128 for \textit{ResNet-18}, \textit{WideResNet-28-10} and \textit{PyramidNet-110}, 512 for \textit{ResNet-50}, 256 for \textit{ViT-S16}} \\
			$\xi$ & - & 0.01\\
			\#epoch & \multicolumn{2}{c}{200 for \textit{ResNet-18} and \textit{WideResNet-28-10}, 300 for \textit{PyramidNet-110}, 90 for \textit{ResNet-50}, 1200 for \textit{ViT-S16}} 
			\\
			\bottomrule
		\end{tabular}
  }
		\label{table:impl-cifar} 
  
		\vskip -.5in
	\end{table}

 \begin{table}[!ht]
		\centering
		\caption{Hyperparameters for \textit{GLUE}.}
		\vskip -.15in
  \resizebox{0.75\textwidth}{!}{
		\begin{tabular}{c c c c c}
			\toprule
			& SAM & LETS-SAM & ASAM & LETS-ASAM \\
			\midrule
			\textit{model parameters} & \\
			dropout rate & $\{0,0.1,0.15\}$ & $\{0,0.1,0.15\}$ & $\{0,0.1,0.15\}$ & $\{0,0.1,0.15\}$ \\
			warmup steps & $\{50, 100, 500, 1000\}$ & $\{50, 100, 500, 1000\}$ & $\{50, 100, 500, 1000\}$ & $\{50, 100, 500, 1000\}$ \\
			learning rate & $\{0,0.1,0.15\}$ & $\{0,0.1,0.15\}$ & $\{0,0.1,0.15\}$ & $\{0,0.1,0.15\}$ \\
            learning rate schedule & Linear & Linear & Linear & Linear \\
			Adam $\epsilon$ & $1e-6$ & $1e-6$ & $1e-6$ & $1e-6$ \\
			Adam $\beta_1$ & $0.9$ & $0.9$ & $0.9$ & $0.9$ \\
			Adam $\beta_2$ & $0.999$ & $0.999$ & $0.999$ & $0.999$ \\
			gradient clipping & $1.0$ & $1.0$ & $1.0$ & $1.0$ \\
			\midrule 
			\textit{perturbation radius $\rho$} & \\
			optimizer  & - & Adam optimizer & - & Adam optimizer\\
			Adam $\beta_1$ & - & $0.9$ &- & $0.9$ \\
			Adam $\beta_2$ & - & $0.999$ & - & $0.999$ \\
   initial learning rate & - & $\{0.0001, 0.0005, 0.001\}$ & - & $\{0.0001, 0.0005, 0.001\}$ \\
			learning rate schedule & & Exponential & & Exponential\\
			\midrule
			batch size & $\{16,32,48,64\}$ & $\{16,32,48,64\}$ & $\{16,32,48,64\}$ & $\{16,32,48,64\}$\\
			$\xi$ & - & - & 0.01 & 0.01
			\\
			\bottomrule
		\end{tabular}
  }
		\label{table:impl-glue} 
		\vskip -.5in
	\end{table}

 \begin{table}[!ht]
		\centering
		\caption{Hyperparameters for \textit{IWSLT’14 DE-EN}.}
		\vskip -.15in
  \resizebox{0.55\textwidth}{!}{
		\begin{tabular}{c c c c c}
			\toprule
			& SAM & LETS-SAM & ASAM & LETS-ASAM \\
			\midrule
			\textit{model parameters} & \\
			dropout rate & $0.3$ & $0.3$ & $0.3$ & $0.3$ \\
			learning rate & $0.0005$ & $0.0005$ & $0.0005$ & $0.0005$\\
			Adam $\beta_1$ & $0.9$ & $0.9$ & $0.9$ & $0.9$ \\
			Adam $\beta_2$ & $0.98$ & $0.98$ & $0.98$ & $0.98$ \\
            weight decay & $0.0001$ & $0.0001$ & $0.0001$ & $0.0001$\\
			\midrule 
			\textit{perturbation radius $\rho$} & \\
			optimizer  & - & Adam optimizer & - & Adam optimizer\\
			Adam $\beta_1$ & - & $0.9$ & - & $0.9$ \\
			Adam $\beta_2$ & - & $0.999$ & - & $0.999$ \\
			initial learning rate & - & $0.00001$ & - & $0.00001$ \\
			learning rate schedule &- & Exponential &- & Exponential\\
			\midrule
			$\xi$ & - & - & 0.01 & 0.01
			\\
   \#epoch & 50 & 50 & 50 & 50
			\\
			\bottomrule
		\end{tabular}
  }
		\label{table:impl-iwslt} 
		\vskip -.5in
	\end{table}


\begin{thebibliography}{8}
\bibitem{allende2013solving}
Allende, G.B., Still, G.: Solving bilevel programs with the {KKT}-approach. Mathematical Programming  (2013)

\bibitem{andriushchenko2022towards}
Andriushchenko, M., Flammarion, N.: Towards understanding sharpness-aware minimization. In: ICML (2022)

\bibitem{bahri2022sharpness}
Bahri, D., Mobahi, H., Tay, Y.: Sharpness-aware minimization improves language model generalization. In: ACL (2022)

\bibitem{bisla2022low}
Bisla, D., Wang, J., Choromanska, A.: Low-pass filtering {SGD} for recovering flat optima in the deep learning optimization landscape. In: AISTATS (2022)

\bibitem{bottou2018optimization}
Bottou, L., Curtis, F.E., Nocedal, J.: Optimization methods for large-scale machine learning. {SIAM} Review  (2018)

\bibitem{bracken1973mathematical}
Bracken, J., McGill, J.T.: Mathematical programs with optimization problems in the constraints. Operations Research  (1973)

\bibitem{cha2021swad}
Cha, J., Chun, S., Lee, K., Cho, H.C., Park, S., Lee, Y., Park, S.: {SWAD}: Domain generalization by seeking flat minima. In: NeurIPS (2021)

\bibitem{dosovitskiy2021an}
Dosovitskiy, A., Beyer, L., Kolesnikov, A., Weissenborn, D., Zhai, X., Unterthiner, T., Dehghani, M., Minderer, M., Heigold, G., Gelly, S., Uszkoreit, J., Houlsby, N.: An image is worth 16x16 words: Transformers for image recognition at scale. In: ICLR (2021)

\bibitem{du2021efficient}
Du, J., Yan, H., Feng, J., Zhou, J.T., Zhen, L., Goh, R.S.M., Tan, V.: Efficient sharpness-aware minimization for improved training of neural networks. In: ICLR (2022)

\bibitem{karolina2017}
Dziugaite, G.K., Roy, D.M.: Computing nonvacuous generalization bounds for deep (stochastic) neural networks with many more parameters than training data. In: UAI (2017)

\bibitem{feurer2019hyperparameter}
Feurer, M., Hutter, F.: Hyperparameter optimization. In: AutoML (2019)

\bibitem{foret2021sharpness}
Foret, P., Kleiner, A., Mobahi, H., Neyshabur, B.: Sharpness-aware minimization for efficiently improving generalization. In: ICLR (2021)

\bibitem{Franceschi2018}
Franceschi, L., Frasconi, P., Salzo, S., Grazzi, R., Pontil, M.: Bilevel programming for hyperparameter optimization and meta-learning. In: ICML (2018)

\bibitem{ghadimi2018approximation}
Ghadimi, S., Wang, M.: Approximation methods for bilevel programming. Preprint arXiv:1802.02246 (2018)

\bibitem{han2017deep}
Han, D., Kim, J., Kim, J.: Deep pyramidal residual networks. In: CVPR (2017)

\bibitem{he2015delving}
He, K., Zhang, X., Ren, S., Sun, J.: Delving deep into rectifiers: Surpassing human-level performance on {ImageNet} classification. In: ICCV (2015)

\bibitem{he2016deep}
He, K., Zhang, X., Ren, S., Sun, J.: Deep residual learning for image recognition. In: CVPR (2016)

\bibitem{he2020deberta}
He, P., Liu, X., Gao, J., Chen, W.: Deberta: Decoding-enhanced bert with disentangled attention. Preprint arXiv:2006.03654 (2020)

\bibitem{hochreiter1994simplifying}
Hochreiter, S., Schmidhuber, J.: Simplifying neural nets by discovering flat minima. In: NeurIPS (1994)

\bibitem{hong2020two}
Hong, M., Wai, H.T., Wang, Z., Yang, Z.: A two-timescale framework for bilevel optimization: Complexity analysis and application to actor-critic. Preprint arXiv:2007.05170 (2020)

\bibitem{izmailov2018averaging}
Izmailov, P., Podoprikhin, D., Garipov, T., Vetrov, D., Wilson, A.G.: Averaging weights leads to wider optima and better generalization. In: UAI (2018)

\bibitem{jiang2021effective}
Jiang, W., Kwok, J., Zhang, Y.: Effective meta-regularization by kernelized proximal regularization. In: NeurIPS (2021)

\bibitem{jiang2022subspace}
Jiang, W., Kwok, J., Zhang, Y.: Subspace learning for effective meta-learning. In: ICML (2022)

\bibitem{jiang2023adaptive}
Jiang, W., Yang, H., Zhang, Y., Kwok, J.: An adaptive policy to employ sharpness-aware minimization. In: ICLR (2023)

\bibitem{jiang2023effective}
Jiang, W., Zhang, Y., Kwok, J.: Effective structured prompting by meta-learning and representative verbalizer. In: ICML (2023)

\bibitem{Jiang2020Fantastic}
Jiang, Y., Neyshabur, B., Mobahi, H., Krishnan, D., Bengio, S.: Fantastic generalization measures and where to find them. In: ICLR (2020)

\bibitem{keskar2017on}
Keskar, N.S., Mudigere, D., Nocedal, J., Smelyanskiy, M., Tang, P.T.P.: On large-batch training for deep learning: Generalization gap and sharp minima. In: ICLR (2017)

\bibitem{khan2018fast}
Khan, M., Nielsen, D., Tangkaratt, V., Lin, W., Gal, Y., Srivastava, A.: Fast and scalable {Bayesian} deep learning by weight-perturbation in {Adam}. In: ICML (2018)

\bibitem{koh2017understanding}
Koh, P.W., Liang, P.: Understanding black-box predictions via influence functions. In: ICML (2017)

\bibitem{krizhevsky2009learning}
Krizhevsky, A., Hinton, G.: Learning multiple layers of features from tiny images. Tech. rep. (2009)

\bibitem{kwon2021asam}
Kwon, J., Kim, J., Park, H., Choi, I.K.: {ASAM}: Adaptive sharpness-aware minimization for scale-invariant learning of deep neural networks. In: ICML (2021)

\bibitem{liao2018reviving}
Liao, R., Xiong, Y., Fetaya, E., Zhang, L., Yoon, K., Pitkow, X., Urtasun, R., Zemel, R.: Reviving and improving recurrent back-propagation. In: ICML (2018)

\bibitem{liu2021value}
Liu, R., Liu, X., Yuan, X., Zeng, S., Zhang, J.: A value-function-based interior-point method for non-convex bi-level optimization. In: ICML (2021)

\bibitem{liu2022auto_lambda}
Liu, S., James, S., Davison, A.J., Johns, E.: {Auto-Lambda}: Disentangling dynamic task relationships. TMLR  (2022)

\bibitem{liu2022towards}
Liu, Y., Mai, S., Chen, X., Hsieh, C.J., You, Y.: Towards efficient and scalable sharpness-aware minimization. In: CVPR (2022)

\bibitem{liu2022random}
Liu, Y., Mai, S., Cheng, M., Chen, X., Hsieh, C.J., You, Y.: Random sharpness-aware minimization. In: NeurIPS (2022)

\bibitem{mi2022make}
Mi, P., Shen, L., Ren, T., Zhou, Y., Sun, X., Ji, R., Tao, D.: Make sharpness-aware minimization stronger: A sparsified perturbation approach. In: NeurIPS (2022)

\bibitem{pedregosa2016hyperparameter}
Pedregosa, F.: Hyperparameter optimization with approximate gradient. In: ICML (2016)

\bibitem{petzka2021relative}
Petzka, H., Kamp, M., Adilova, L., Sminchisescu, C., Boley, M.: Relative flatness and generalization. In: NeurIPS (2021)

\bibitem{qu2022generalized}
Qu, Z., Li, X., Duan, R., Liu, Y., Tang, B., Lu, Z.: Generalized federated learning via sharpness aware minimization. In: ICML (2022)

\bibitem{reddi2016stochastic}
Reddi, S.J., Hefny, A., Sra, S., Poczos, B., Smola, A.: Stochastic variance reduction for nonconvex optimization. In: ICML (2016)

\bibitem{russakovsky2015imagenet}
Russakovsky, O., Deng, J., Su, H., Krause, J., Satheesh, S., Ma, S., Huang, Z., Karpathy, A., Khosla, A., Bernstein, M., Berg, A.C., Fei-Fei, L.: {ImageNet} large scale visual recognition challenge. IJCV  (2015)

\bibitem{sinha2019using}
Sinha, A., Soun, T., Deb, K.: Using {Karush-Kuhn-Tucker} proximity measure for solving bilevel optimization problems. Swarm and Evolutionary Computation  (2019)

\bibitem{stadie2020learning}
Stadie, B., Zhang, L., Ba, J.: Learning intrinsic rewards as a bi-level optimization problem. In: UAI (2020)

\bibitem{vaswani2017attention}
Vaswani, A., Shazeer, N., Parmar, N., Uszkoreit, J., Jones, L., Gomez, A.N., Kaiser, {\L}., Polosukhin, I.: Attention is all you need. In: NeurIPS (2017)

\bibitem{wang2018glue}
Wang, A., Singh, A., Michael, J., Hill, F., Levy, O., Bowman, S.R.: {GLUE}: A multi-task benchmark and analysis platform for natural language understanding. Preprint arXiv:1804.07461 (2018)

\bibitem{ye2021multiobjective}
Ye, F., Lin, B., Yue, Z., Guo, P., Xiao, Q., Zhang, Y.: Multi-objective meta learning. In: NeurIPS (2021)

\bibitem{zagoruyko2016wide}
Zagoruyko, S., Komodakis, N.: Wide residual networks. In: BMVC (2016)

\bibitem{zhang2021understanding}
Zhang, C., Bengio, S., Hardt, M., Recht, B., Vinyals, O.: Understanding deep learning (still) requires rethinking generalization. Communications of the ACM  (2021)

\bibitem{zhao2022penalizing}
Zhao, Y., Zhang, H., Hu, X.: Penalizing gradient norm for efficiently improving generalization in deep learning. In: ICML (2022)

\bibitem{zhao2022ss}
Zhao, Y., Zhang, H., Hu, X.: Randomized sharpness-aware training for boosting computational efficiency in deep learning. Preprint arXiv:2203.09962 (2022)

\bibitem{zhou2019toward}
Zhou, M., Liu, T., Li, Y., Lin, D., Zhou, E., Zhao, T.: Toward understanding the importance of noise in training neural networks. In: ICML (2019)

\bibitem{zhuang2022surrogate}
Zhuang, J., Gong, B., Yuan, L., Cui, Y., Adam, H., Dvornek, N.C., S.~Duncan, J., Liu, T.: Surrogate gap minimization improves sharpness-aware training. In: ICLR (2022)
\end{thebibliography}
\end{document}